\documentclass{article}

\usepackage[preprint]{neurips_2026}
\usepackage{amsmath}
\usepackage{amssymb}
\usepackage{mathtools}
\usepackage{amsthm}
\usepackage{microtype}
\usepackage{graphicx}
\usepackage{subcaption}
\usepackage{booktabs} % for professional tables

\usepackage{mathrsfs}
\usepackage{amssymb}
\usepackage{bm}
\usepackage{amsmath}
\usepackage{dsfont}
\usepackage{booktabs}
\usepackage{epstopdf}
\usepackage{multirow}
\usepackage{comment}
\usepackage{endnotes}
\usepackage{hyperref}

\usepackage{bibentry}
\usepackage[normalem]{ulem}
\usepackage{graphicx}
\usepackage{enumerate}
\usepackage{enumitem}
\usepackage[most]{tcolorbox}
\usepackage{subcaption}
\usepackage{etoc}
\usepackage{xcolor,xspace}
\definecolor{mydarkred}{rgb}{0.6,0,0}
\definecolor{myblue}{HTML}{268BD2}
\usepackage{inconsolata}
\usepackage[table]{xcolor}
\newcommand{\ourmodel}{{\texttt{\textbf{TrajEvolve}}}\xspace}
\usepackage{hyperref}
\hypersetup{
  colorlinks=true,
  linkcolor=mydarkred,
  citecolor=myblue
}
\usepackage{multirow}
\usepackage{wrapfig}
\usepackage{booktabs}
\usepackage{graphicx}
\usepackage[capitalize,noabbrev]{cleveref}
\usepackage{algorithm}
\usepackage{algorithmic}
\usepackage{thmtools}
\usepackage{thm-restate}

\usepackage{xcolor}
\usepackage{booktabs}
\usepackage{makecell}
\definecolor{NecessaryColor}{HTML}{C0392B}    % red, matches "Ours" in our figures
\definecolor{RedundantColor}{HTML}{7F7F7F}    % gray
\newcommand{\necTag}{\textcolor{NecessaryColor}{\textbf{Necessary}}}
\newcommand{\redTag}{\textcolor{RedundantColor}{\textit{Redundant}}}

\theoremstyle{plain}
\newtheorem{theorem}{Theorem}[section]
\newtheorem{proposition}[theorem]{Proposition}
\newtheorem{lemma}[theorem]{Lemma}

\theoremstyle{definition}
\newtheorem{definition}[theorem]{Definition}

\theoremstyle{remark}
\newtheorem{remark}[theorem]{Remark}

\usepackage[utf8]{inputenc} % allow utf-8 input
\usepackage[T1]{fontenc}    % use 8-bit T1 fonts
\usepackage{hyperref}       % hyperlinks
\usepackage{url}            % simple URL typesetting
\usepackage{booktabs}       % professional-quality tables
\usepackage{amsfonts}       % blackboard math symbols
\usepackage{nicefrac}       % compact symbols for 1/2, etc.
\usepackage{microtype}      % microtypography
\usepackage{xcolor}         % colors
\usepackage{graphicx}
\usepackage{hyperref}

\title{TrajEvolve: Trajectory Evolution for Reinforcement Learning 
     with Hindsight Credit Assignment}

\author{%
  Yi Wen$^{1}$\thanks{Both authors contributed equally to this work.},
  Hao Chen$^{2}$\footnotemark[1],
  Wanyu Wang$^{1}$, Maolin Wang$^{1}$, Pengyue Jia$^{1}$, Derong Xu$^{1}$, \\
  \textbf{Huize Tan$^{3}$, Yingyi Zhang$^{1}$, Wenlin Zhang$^{1}$,
  Weihong Luo$^{2}$, Xiku Du$^{2}$, Xiangyu Zhao$^{1}$} \\
  $^{1}$City University of Hong Kong, $^{2}$Tencent, $^{3}$Nanjing University \\
  \texttt{yiwen23-c@my.cityu.edu.hk},
  \texttt{xianzhao@cityu.edu.hk}
}

\begin{document}

\maketitle

\begin{abstract}
The Reinforcement Learning (RL) paradigm has
achieved remarkable success in enhancing the reasoning capabilities of Large Language Models
(LLMs). However, when applied to long-horizon tasks, it faces severe challenges: i) Due to the high task complexity, the model struggles to sample enough successful trajectories, resulting in slow convergence. ii) Relying solely on sparse final rewards causes all steps to be updated indiscriminately, encouraging redundant step generation. When both the quantity and quality of successful trajectories become difficult to guarantee, the performance is severely constrained. This dilemma stems from the limitations: existing RL paradigms rely entirely on the reasoning capability of base models, lacking the ability to actively optimize trajectories. To address the above issues, we propose a novel trajectory evolution paradigm with hindsight credit assignment, termed \ourmodel. TrajEvolve not only relies on the trajectory inferred by the model but also generates ``evolutionary trajectories'' by refining low-quality steps in historical trajectories, thereby  exposing the policy to high-quality compositions that lie within the base model's support but are rarely sampled within practical training budgets.  Furthermore, TrajEvolve reformulates step-level continuous credit assignment as a binary classification task, whether each step is necessary for task completion. This simplification not only reduces the learning difficulty but also enables step labels to be obtained through counterfactual verification, achieving precise identification of redundant steps. Extensive experiments on two datasets demonstrate the superiority of the proposed method\footnote{\url{https://github.com/wenyiwy99/TrajEvolve}}.
\end{abstract}

\vspace{-0.5 em}
\section{Introduction}

Large Language Models (LLMs)~\citep{gpt4,qwen3} have achieved remarkable breakthroughs in language understanding and generation. Among post-training techniques for LLMs, Reinforcement Learning (RL), which iteratively improves models through exploration and reward-driven optimization, has emerged as one of the most influential paradigms. Group Relative Policy Optimization~\citep{rl_grpo} and its variants, which generate multiple rollout trajectories and leverage inter-group reward comparisons for policy optimization, have gained significant attention.

Despite great success achieved, RL paradigms face severe challenges when applied to long-horizon decision-making tasks~\citep{rl_gigpo,alfworld}. Such tasks typically require the model to interact with external environments over dozens of steps, with reward signals provided only at the final outcome. The impact of any individual action may not manifest until many steps later, substantially increasing the difficulty of reinforcement learning. They face two typical dilemmas in this setting: i) Due to high task complexity and vast exploration space, the model struggles to sample sufficient successful trajectories, resulting in sparse learning signals and slow convergence. ii) Relying solely on sparse final rewards causes all steps to be updated indiscriminately, encouraging the generation of redundant steps in successful trajectories. When neither the quantity nor the quality of trajectories can be guaranteed, the performance of RL paradigms is severely constrained.

These issues fundamentally stem from a critical limitation of existing RL paradigms: \textit{\uline{their performance is bounded by the base model's reasoning capability, lacking the ability to actively optimize trajectories}}. Recent studies~\citep{rl_best} have pointed out that RL merely improves the sampling efficiency of correct reasoning paths already within the base model's repertoire, rather than eliciting fundamentally new reasoning patterns. This raises a natural question: 
\definecolor{cvprblue}{rgb}{0.21,0.49,0.74}
\begin{tcolorbox}[colframe=black!50, colback=cvprblue!8, boxrule=1.5pt, arc=2mm, top=4pt, bottom=4pt, left=4pt, right=4pt,  boxsep=1pt]
\raisebox{-0.2\baselineskip}{\includegraphics[height=1\baselineskip]{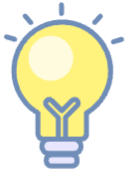}}
\textit{Can we proactively construct higher-quality trajectories during training to help models to overcome the sampling barrier that on-policy RL inherently faces in long-horizon tasks?}
\end{tcolorbox}

The key challenge in addressing this problem~\citep{rl_arpo} lies in accurately identifying the valuable steps, which is essentially a step-level credit assignment problem. Existing credit assignment methods can be categorized into two types: i) Heuristic-based signals~\citep{rl_empg}, which rely on indicators such as LLM output confidence that have no direct correlation with task outcomes. Such signals may reinforce confident but incorrect outputs. ii) Outcome-based signals~\citep{rl_gigpo, rl_tree-grpo}, which estimate credit through multiple rollouts on intermediate steps, evaluating credit with final rewards. These methods focus solely on which steps are likely to lead to success, but such biased estimations cannot identify steps that are redundant yet do not affect the final outcome.  In summary, existing methods struggle to determine the true contribution of individual steps to the task, let alone identify redundant steps in the success trajectory. Moreover, accurate credit assignment remains extremely challenging given the vast solution space of long-horizon tasks.

To address these challenges, we propose \ourmodel, a novel trajectory evolution for reinforcement learning with hindsight credit assignment. TrajEvolve not only relies on the trajectory inferred by the model but also generates ``evolutionary trajectories'' by refining low-quality steps in historical trajectories, thereby exposing the policy to high-quality compositions that lie within the base model's support but are rarely sampled within practical training budgets. Notably, TrajEvolve reformulates continuous credit assignment as a binary classification of whether each step is necessary for task. This simplification compresses the infinite solution space into a finite subspace and can be verified through counterfactual experiments, thereby reducing learning difficulty and ensuring sufficient labeled data for Discriminator training.  Specifically, TrajEvolve obtains step-level redundancy labels through hindsight verification and trains a discriminator model to evolve historical trajectories. The evolved trajectories serve as high-quality samples and are jointly utilized with policy-generated trajectories for reinforcement learning optimization.  The main contributions of the paper are as follows:
\vspace{-0.5em}
\begin{itemize}[leftmargin=*]
    \item We propose a trajectory evolution paradigm that not only trains on the model-inferred trajectory but also generates ``Evolutionary Trajectories'' by refining low-quality steps in historical trajectories, thereby exposing the policy to high-quality compositions that are rarely sampled within training.
    \item We reformulate continuous credit assignment as a binary classification of whether each step is necessary for task completion. This simplification compresses the infinite solution space into a verifiable subspace, substantially reducing difficulty and achieving effective credit assignment.
    \item Experiments on ALFWorld and WebShop validate TrajEvolve's effectiveness.

\end{itemize}
\vspace{-1.3em}

% \begin{figure}[t]
% \centering
% \includegraphics[width=0.99\columnwidth]{fig/compare.pdf}
% \caption{ .}
% \label{fig:intro}
% \end{figure}

 \begin{figure*}[t]
\centering
\setlength{\abovecaptionskip}{0pt}
\includegraphics[width=1\linewidth]{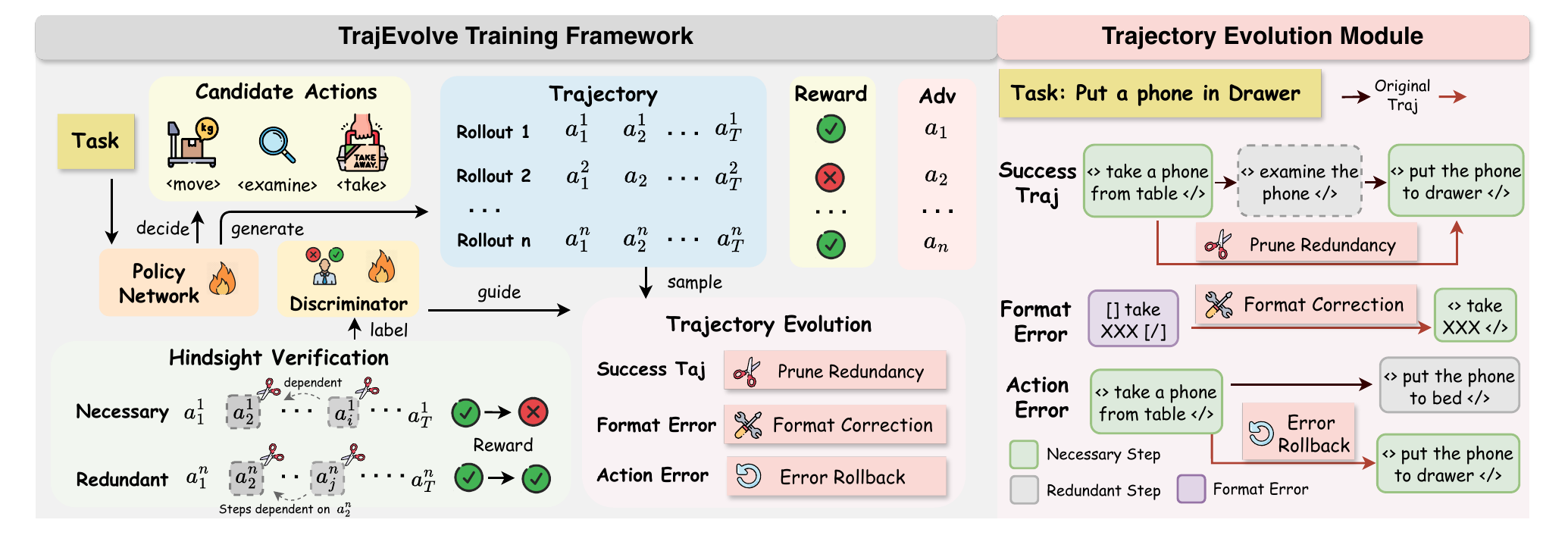}
\caption{Overview of the TrajEvolve training framework. \textbf{Left:} Training pipeline with hindsight verification and trajectory evolution module. \textbf{Right:} Three trajectory evolution strategies.}
\label{fig:intro}
\vspace{-1.3em}
\end{figure*}
\section{Preliminaries}
\vspace{-0.5em}
\textbf{Problem Setup.} Following~\citep{rl_gigpo}, we consider a general multi-step decision-making setting where an LLM agent interacts with an environment based on a task description $x$. At each time step $t = 1, 2, \ldots, T$, the agent observes a state $s_t = (o_t, \mathcal{A}_t)$, comprising an observation $o_t$ and the set of available actions $\mathcal{A}_t$, then selects an action $a_t \in \mathcal{A}_t$ according to its policy $\pi_\theta(a_t \mid o_t, \mathcal{A}_t)$. The environment subsequently transitions to state $s_{t+1}$. A complete $T$-step episode forms a trajectory $\tau = \{(s_1, a_1), (s_2, a_2), \ldots, (s_T, a_T)\}$. Notably, only a sparse binary reward $r \in \{0, 1\}$ is provided at episode termination in our setting,  indicating task success or failure. 
% This sparse and delayed feedback makes credit assignment particularly challenging in long-horizon tasks.

\textbf{Group-based RL.} Group-based reinforcement learning~\citep{rl_grpo} bypasses auxiliary critic networks by computing advantages through within-group normalization. Given a task $x$, the policy generates $N$ candidate trajectories $\{\tau_1, \ldots, \tau_N\}$ with scalar rewards $R(\tau_i)$. Advantages are computed as:$A(\tau_i) = \frac{R(\tau_i) - \mu}{\sigma}$, 
where $\mu$ and $\sigma$ denote the group mean and standard deviation, respectively. This yields a memory-efficient and scalable framework. 

\textbf{Credit Assignment.} Group-based methods assign uniform credit to all steps based solely on terminal rewards, which becomes problematic where steps contribute unequally. Credit assignment addresses this by deriving step-level scores:
\begin{equation}
    c_t = g(s_t, a_t, \tau, R(\tau)), \quad t = 1, \ldots, T.
\end{equation}
where $c_t$ is the credit score for step $t$, $s_t$ and $a_t$ denote the state and action, $\tau$ represents the complete trajectory, and $R(\tau)$ is the terminal reward. However, existing methods attempt to estimate precise credit scores for every step, yet the vast solution space inherent to long-horizon tasks renders exact credit quantification exceedingly difficult.

 \vspace{-0.5 em}

\section{TrajEvolve}
\vspace{-0.5em}
In this section, we present \ourmodel for step-level credit assignment in long-horizon agent tasks. We begin by reformulating the credit assignment problem from continuous score estimation to binary necessity classification, which substantially reduces the solution space and enables reliable label acquisition (\S\ref{sec:reformulation}). We then describe a hindsight verification procedure that obtains ground-truth necessity labels through counterfactual execution in the environment (\S\ref{sec:verification}). Next, we integrate the obtained labels into policy optimization via redundancy-aware reward shaping and step-level advantage adjustment (\S\ref{sec:optimization}). Finally, we introduce a trajectory evolution mechanism that refines historical trajectories by eliminating redundant steps and correcting errors, enabling the model to learn from improved trajectories beyond its current reasoning capability (\S\ref{sec:evolution}).
\vspace{-5 pt}

\subsection{Problem Reformulation: Continuous to Discrete}
\vspace{-5 pt}
\label{sec:reformulation}
%  given a trajectory $\tau = \{(s_t, a_t)\}_{t=1}^{T}$, we define:
% \begin{equation}
%     y_t = \mathbb{I}\left[\mathcal{V}(\tau \setminus \{(s_t, a_t)\}) < \mathcal{V}(\tau)\right], \quad t = 1, \ldots, T,
% \end{equation}
% where $\mathcal{V}(\cdot)$ denotes the task success indicator and $y_t \in \{0, 1\}$ represents the necessity label for step $t$.

Existing credit assignment methods either rely on predefined scoring approaches or train implicit process reward models~\citep{prime,rl_istar} to estimate step-level contributions. However, these approaches fundamentally attempt to predict precise continuous credit values over an infinite solution space, which becomes intractable for long-horizon tasks. Moreover, since supervision signals are only derived from final results, the feedback is inherently delayed, making it difficult to accurately attribute credit to individual steps.

We observe that for effective policy learning, the critical signal is not the precise credit magnitude, but whether each step contributes positively to task success. This is particularly pronounced in long-horizon tasks, where exploratory interactions with the environment introduce numerous redundant steps that outcome-based methods fail to properly process. This insight motivates us to reformulate credit assignment as a binary \textit{necessary} or \textit{redundant} classification problem. Formally, given a trajectory $\tau = \{(s_t, a_t)\}_{t=1}^{T}$, we define a step-wise  discriminator $\mathcal{D}_\phi: \mathcal{S} \times \mathcal{A} \times \mathcal{T} \rightarrow \{0, 1\}$:
\begin{equation}
    \hat{y}_t = \mathcal{D}_\phi(s_t, a_t, \tau),
\end{equation}
where $\hat{y}_t = 1$ indicates that step $t$ is predicted as \textit{necessary}, and $\hat{y}_t = 0$ as \textit{redundant}.

We emphasize the superiority of reformulation as follows: i) \textit{\uline{it compresses the infinite solution space from continuous values to a finite hypothesis space}}, substantially reducing learning complexity and enabling even small models to achieve strong performance. ii) \textit{\uline{Unlike continuous credit scores whose ground truth is practically unobtainable, binary credit labels can be reliably acquired through hindsight verification}}, ensuring sufficient supervision for training the discriminator to identify redundant versus necessary steps. We detail the verification procedure in the next subsection.
\vspace{-0.5 em}

\subsection{Hindsight Verification on Environment-Changing Steps}
\label{sec:verification}

A key advantage of the reformulation is that credit labels can be reliably acquired through hindsight verification in the environment. In this section, we describe how to obtain ground-truth credit labels.

\textbf{Action Categorization.} We distinguish two types of actions in 
environment-interactive tasks: \textit{environment-changing actions} that 
alter the environment state (e.g., \textit{take}, \textit{move}), and 
\textit{information-gathering actions} that only query the environment 
without modifying it (e.g.,  \textit{examine}). 
Hindsight verification is well-defined only for the former: 
removing an information-gathering step neither breaks downstream 
feasibility nor alters task outcomes, so its necessity cannot be 
determined by environment replay. We therefore apply hindsight 
verification exclusively to environment-changing steps and assign 
$y_t = 1$ to all information-gathering steps by default. In ALFWorld 
and WebShop, environment-changing actions account for the vast 
majority of steps; the extension to environments with substantial 
information-gathering steps (e.g., ScienceWorld) is discussed in 
Appendix~\ref{sec:scienceworld}.

To determine whether a step is necessary, we must account for its downstream effects.

\begin{definition}[Dependent Steps]
    Given a successful trajectory $\tau = \{(s_t, a_t)\}_{t=1}^{T}$ and a step index $t$, we define the \textit{dependent step set} $\mathcal{D}(t)$ as the set of subsequent steps whose actions become infeasible when step $t$ is removed:
    \begin{equation}
    \mathcal{D}(t) = \left\{ j : j > t, \; a_j \notin \mathcal{A}(s_j^{-t}) \right\},
\end{equation}
where $s_j^{-t}$ denotes the state at step $j$ when step $t$ is skipped, and $\mathcal{A}(s)$ is the set of admissible actions at state $s$.
\end{definition}

\begin{remark}
Accurate computation of the dependent step set $\mathcal{D}(t)$ is critical for subsequent credit labeling. Errors in identifying dependent steps can cause information leakage or incorrect credit labels inheritance via Proposition \ref{pro:transitivity}, ultimately misleading the entire training process.
\end{remark}

\begin{remark}
    In environment-interactive scenarios, identifying dependent steps is straightforward. If removing step $t$ causes a subsequent action to become infeasible, that action is dependent on step $t$. Action feasibility can be directly read from environment feedback—task execution environments such as ALFWorld explicitly return an execution-status signal, making this judgment trivial.
\end{remark}

Besides the aforementioned definitions, we also require the following properties to determine steps:

\begin{proposition}[Necessity]
\label{pro:necessity}
Step $t$ is \textbf{necessary} only if removing step $t$ along with its dependent steps causes task failure in a successful trajectory $\tau$:
\begin{equation}
    \mathcal{V}\left(\tau \setminus \left(\{t\} \cup \mathcal{D}(t)\right)\right) = 0,
\end{equation}
where $\mathcal{V}(\cdot) \in \{0, 1\}$ is the task success indicator. 
\end{proposition}

\begin{remark}
The converse of Proposition~\ref{pro:necessity} does not hold: task failure upon removal does not imply the step is necessary Consider a trajectory: \textit{go to cabinet} $\rightarrow$ \textit{open cabinet} $\rightarrow$ \textit{go to bathroom} $\rightarrow$ \textit{go to cabinet} $\rightarrow$ \textit{take clothes}. Both \textit{go to bathroom} and the second \textit{go to cabinet} are redundant, yet removing the latter alone still causes failure. 
\end{remark}
%: 

\begin{proposition}[Sufficiency]
\label{pro:sufficiency}
Step $t$ is \textbf{redundant} if removing step $t$ along with its dependent steps preserves task success in a successful trajectory $\tau$:
\begin{equation}
    \mathcal{V}\left(\tau \setminus \left(\{t\} \cup \mathcal{D}(t)\right)\right) = 1.
\end{equation}
 
\end{proposition}

\begin{proposition}[Transitivity]
\label{pro:transitivity}
The dependent steps of a redundant step are all redundant: if $y_t = 0$, then $y_j = 0$ for all $j \in \mathcal{D}(t)$.
\end{proposition}

\begin{remark}
The dependent steps of a necessary step are not necessarily necessary.
\end{remark}

% \textbf{Forward Ablation Algorithm.} Based on these properties, we design a forward ablation procedure to efficiently label all steps. Let $\mathcal{R}$ denote the set of confirmed redundant steps, initialized as $\mathcal{R} = \emptyset$. For each step $t = 1, \ldots, T$ in order:
% \begin{enumerate}
%     \item If $t \in \mathcal{R}$, skip (already labeled as redundant).
%     \item Execute the modified trajectory $\tau \setminus (\{t\} \cup \mathcal{R})$ in the environment.
%     \item Compute $\mathcal{D}(t)$ by identifying steps with infeasible actions.
%     \item If task succeeds: $y_t = 0$, and update $\mathcal{R} \leftarrow \mathcal{R} \cup \{t\} \cup \mathcal{D}(t)$.
%     \item If task fails: $y_t = 1$ (step $t$ is necessary).
% \end{enumerate}

% This procedure requires at most $T$ environment rollouts per trajectory, and leverages Proposition 2 to avoid redundant verification of dependent steps.

\textbf{Discriminator Training.} With the credit labels $\{y_t\}_{t=1}^{T}$ obtained via Algorithm~\ref{alg:forward_ablation}, we train the discriminator $\mathcal{D}_\phi$ via supervised fine-tuning on an LLM backbone. Given a prompt $\mathbf{x}_t$ encoding the trajectory context and the target step $(s_t, a_t)$, the model outputs a binary judgment. The training objective minimizes the cross-entropy loss:
\begin{equation}
    \mathcal{L}_{\text{SFT}} = -\mathbb{E}_{(\mathbf{x}_t, y_t) \sim \mathcal{D}_{\text{train}}} \left[ \log P_\phi(\hat{y}_t = y_t \mid \mathbf{x}_t) \right],
\end{equation}
where $\mathcal{D}_{\text{train}}$ denotes the training set constructed from verified trajectories.

\begin{wrapfigure}{r}{0.6\linewidth}
\vspace{0em}
\centering
\footnotesize
\captionsetup{aboveskip=0pt, belowskip=2pt}

\hrule height 0.8pt
\vspace{2pt}

\refstepcounter{algorithm}
\label{alg:forward_ablation}
{\raggedright
\noindent\textbf{Algorithm~\thealgorithm:} Hindsight Verification for Credit Labeling\par}

\vspace{1pt}
\hrule height 0.4pt
\vspace{2pt}

\begin{algorithmic}[1]
\REQUIRE Successful Trajectory $\tau$, Environment $\mathcal{E}$
\ENSURE Credit Labels $\{y_t\}_{t=1}^{T}$
\STATE Initialize Redundant Set $\mathcal{R} \leftarrow \emptyset$
\FOR{$t = 1$ to $T$}
    \IF{$t \in \mathcal{R}$}
        \STATE $y_t \leftarrow 0$ \hfill $\triangleright$ Proposition~\ref{pro:transitivity}
        \STATE \textbf{continue}
    \ENDIF
    \STATE Compute $\mathcal{D}(t) \leftarrow \{j > t : a_j \notin \mathcal{A}(s_j^{-t})\}$
    \STATE Execute $\tau \setminus (\{t\} \cup \mathcal{R} \cup \mathcal{D}(t))$ in $\mathcal{E}$
    \IF{task succeeds}
        \STATE $y_t \leftarrow 0$ \hfill $\triangleright$ Proposition~\ref{pro:sufficiency}
        \STATE $\mathcal{R} \leftarrow \mathcal{R} \cup \{t\} \cup \mathcal{D}(t)$ \hfill $\triangleright$ Proposition~\ref{pro:transitivity}
    \ELSE
        \STATE $y_t \leftarrow 1$ \hfill $\triangleright$ Proposition~\ref{pro:necessity}
    \ENDIF
\ENDFOR
\end{algorithmic}

\vspace{2pt}
\hrule height 0.8pt
\vspace{-1em}
\end{wrapfigure}
\textbf{Discriminator Pre-training.} During early RL training, the policy model struggles to generate sufficient successful trajectories, resulting in sparse credit labels. To address this cold-start problem, we pre-train the discriminator before RL optimization. Specifically, we use Qwen2.5-14B-Instruct to generate approximately 2,000 successful trajectories, which are processed through Algorithm~\ref{alg:forward_ablation} to obtain credit labels.

Notably, unlike conventional SFT that requires high-quality trajectories from strong teacher models, our paradigm does not require optimal or concise trajectories. Trajectories with  redundant steps provide richer negative samples. This property reduces dependency on powerful LLMs.
\vspace{-0.5 em}

\subsection{Policy Optimization with Credit Labels}
\label{sec:optimization}
Building upon the credit labels, we integrate step-level credit into the policy optimization.

\textbf{Redundancy-Aware Reward.}
Conventional methods often apply uniform penalties based on trajectory length or repeated actions, which may inadvertently discourage longer but correct trajectories and lead to suboptimal performance. To address this limitation, we propose a redundancy-aware reward mechanism. Specifically, for each successful trajectory $\tau_i$, the shaped reward is computed as:
\begin{equation}
\label{equation:reward}
\small
    \tilde{R}(\tau_i) = R(\tau_i) - \alpha \cdot \rho(\tau_i),
\end{equation}
where $R(\tau_i)$ is the original task reward, $\rho(\tau_i)$ denotes the redundancy ratio defined as the proportion of redundant steps in the trajectory, and $\alpha \geq 0$ controls the penalty strength. 

Instead of penalizing indiscriminately, our formulation applies targeted penalties only to trajectories with genuinely redundant steps, thereby preserving high rewards for efficient executions.

\textbf{Step-Level Advantage.}
Given a task $x$, the policy $\pi_{\theta_{\text{old}}}$ generates a group of $N$ complete trajectories $\{\tau_1, \tau_2, \ldots, \tau_N\}$ under identical initial conditions. Each trajectory $\tau_i$ receives the shaped reward $\tilde{R}(\tau_i)$ as described above. To evaluate the relative quality of each trajectory within the group, we compute an episode-level advantage $A^E(\tau_i)$ as follows:
\begin{equation}
\label{equation:grpo_adv}
\small
    A^E(\tau_i) = \frac{\tilde{R}(\tau_i) - \text{mean}(\{\tilde{R}(\tau_j)\}_{j=1}^{N})}{\text{std}(\{\tilde{R}(\tau_j)\}_{j=1}^{N})}.
\end{equation}

While the episode-level advantage captures trajectory-wide performance, it assigns uniform credit to all steps. To provide finer-grained guidance, we introduce a step-level advantage $A^S(a_t^{(i)})$ derived from the discriminator predictions. The final advantage for each action combines both levels:
\begin{equation}
\label{equation:trajevolve}
    A(a_t^{(i)}) = A^E(\tau_i) + \lambda \cdot A^S(a_t^{(i)}),
\end{equation}
where $\lambda > 0$ is a weighting coefficient, and $A^S(a_t^{(i)}) = \mathbb{I}[y_t = 1]$ indicates whether step $t$ is necessary. For failed trajectories, we encourage the correct prefix by applying the bonus to steps preceding the first erroneous step $t_{\textnormal{err}}$, i.e., the first step that is either invalid or predicted as redundant by the discriminator. This hierarchical advantage structure provides both trajectory-wide and step-wise supervision, enabling more effective credit assignment in long-horizon tasks.

\textbf{Policy Update.}
We optimize the policy as follows:
\begin{equation}
\label{equation:policy}
    \mathcal{J}(\theta) = \mathbb{E}\Big[ \frac{1}{T} \sum_{t=1}^{T} \min( 
         \rho_\theta(a_t)  A(a_t), \text{clip}(\rho_\theta(a_t), 1 \pm \epsilon) A(a_t)
    ) \Big],
\end{equation}
where $\rho_\theta(a_t) = \frac{\pi_\theta(a_t | s_t)}{\pi_{\theta_{\text{old}}}(a_t | s_t)}$ is the importance sampling ratio.

\vspace{-0.5em}
\subsection{Trajectory Evolution}
\label{sec:evolution}

Existing RL paradigms' performance~\citep{rl_rlver} is inherently bounded by the base model's reasoning capability. In this section, we introduce a trajectory evolution method that actively refines historical trajectories during training.

\textbf{Evolution Strategies.}
Let $\mathcal{H} = \{\tau_1, \tau_2, \ldots\}$ denote the historical trajectory sampled during training. Based on the step-level credit labels obtained, we apply different evolution strategies:
\vspace{-0.3em}
\begin{itemize}[leftmargin=*]
    \item \textbf{Redundancy Elimination}: Given a successful trajectory $\tau$ with credit labels $\{y_t\}_{t=1}^T$, we construct an evolved trajectory $\tau^*$ by retaining only the necessary steps:
    $$
        \tau^* = \{(s_t, a_t) \in \tau \mid y_t = 1\}.
    $$
    \item \textbf{Format Correction}: When a trajectory fails due to an invalid action format at step $t_{\text{err}}$, we map the erroneous action to the nearest valid action in the admissible action space $\mathcal{A}(s_{t_{\text{err}}})$.
    
    \item \textbf{Error Rollback}: When a trajectory fails due to an incorrect action at step $t_{\text{err}}$, we retain all steps before $t_{\text{err}}$ and allow the policy to regenerate, providing opportunities to explore alternative actions.
\end{itemize}

\textbf{Training with Evolved Trajectories.}
The evolved trajectories are re-executed in the environment to obtain updated rewards, and subsequently incorporated into the policy optimization alongside newly sampled trajectories. This augmentation enriches the training distribution with high-quality, non-redundant trajectories that may not be sampled by the current policy, thus   exposing the policy to high-quality compositions that are rarely sampled within practical training budgets.

\section{Theoretical Analysis}
\label{sec:theory}
 \vspace{-0.5 em}
In this section, we provide a theoretical analysis showing that, within a stylized model of group-relative policy optimization, on-policy RL faces a redundancy lower bound in long-horizon tasks, and that trajectory evolution can lift this barrier.
  \vspace{-0.5 em}
  
\paragraph{Problem Setup.}
In this section, we consider the classic GRPO training framework. At each iteration, the policy $\pi_\theta$ generates a group of $N$ trajectories for each task. Each trajectory $\tau = \{(s_t, a_t)\}_{t=1}^T$ receives a shaped reward $\tilde{R}(\tau) = R(\tau) - \alpha \cdot \rho(\tau)$, where $\rho(\tau)$ denotes the redundancy ratio. We define a \emph{clean trajectory} as a successful trajectory with zero redundancy, i.e., $R(\tau)=1, \rho(\tau)=0$.
 
\begin{restatable}[Per-Step Redundancy Model]{assumption}{redundancyasmp}
\label{asmp:redundancy}
There exists $\beta \in (0,1]$ such that, along every rollout, a $\beta$-fraction of the states $s_t$ are \emph{redundancy-prone}: under the policy $\pi_{\theta_k}$ at iteration $k$, the action taken at each such state is redundant independently with probability $p_k$, i.e.,
\begin{equation}
    P_{\pi_{\theta_k}}(r_t = 1 \mid s_t) = p_k,
\end{equation}
while actions at the remaining states are never redundant, where $r_t$ is the redundancy indicator, $r_t=1$ if step $t$ is redundant and $r_t=0$ if necessary. For the initial policy $\pi_{\theta_0}$, $p_0 > 0$.
\end{restatable}

\begin{remark}
The condition $p_0 > 0$ is mild since the LLM-based policy outputs a softmax distribution over all admissible actions, where every action retains positive probability, and in long-horizon interactive tasks multiple plausible but unnecessary actions exist at most states; $\beta$ accounts for the states at which no redundant action is available. 
\end{remark}

% This assumption is reasonable for LLM-based policies. Since the policy outputs a softmax distribution over all admissible actions, every action---including redundant ones such as revisiting explored locations or repeating queries---retains positive probability. In long-horizon interactive tasks, multiple plausible but unnecessary actions exist at most states, making $p_0 > 0$ a mild condition.

\begin{restatable}[Gradient Smoothness]{assumption}{smoothness}
\label{asmp:smooth}
Let $m_{\textnormal{clean}}^{(k)}$ denote the number of clean trajectories in the training group at iteration $k$. There exists a bounded, positive, non-decreasing function $h:\{1,\dots,N-1\}\to(0,h_{\max}]$, extended by $h(0)=h(N)=0$, such that, conditional on the group composition, the expected per-iteration reduction of the per-step redundancy probability $p_k$ satisfies
\begin{equation}
E\big[\,p_k - p_{k+1} \mid m_{\textnormal{clean}}^{(k)} = m\,\big] = h(m)\,\eta,
\qquad 0 \le m \le N.
\end{equation}
\end{restatable}

\begin{remark}
This assumption captures two properties: (i) the total gradient effect is bounded by $h_{\max}\eta$ per iteration, enforced by gradient clipping and learning rates; (ii) the gradient signal for reducing redundancy is governed by the number of clean trajectories, since clean trajectories provide a direct comparison between redundant and non-redundant behavior in the advantage computation. 
\end{remark}

\begin{restatable}[Redundancy Lower Bound for On-Policy RL]{theorem}{redundancy}
\label{thm:main}
Under Assumptions~\ref{asmp:redundancy} and~\ref{asmp:smooth}, for any training iteration $K \leq K^*$, the expected redundancy ratio of the trained policy satisfies:
\begin{equation}
\label{eq:main_bound}
    E\!\left[\rho(\tau_K)\right] \geq  \frac{\beta p_0}{2},
\end{equation}
where $K^* = p_0 / \big(2 N h_{\max}\,\eta\,(1 - p_0/2)^{\beta T}\big)$ and $T$ is the length of trajectory $\tau$.
\end{restatable}

\begin{remark}
As the task horizon $T$ increases, $K^*$ grows exponentially since its denominator $(1 - p_0/2)^{\beta T}$ vanishes, meaning that even reaching the redundancy floor $\beta p_0/2$ requires exponentially more iterations for longer-horizon tasks. 
\end{remark}

\begin{restatable}{proposition}{trajevolve}
\label{prop:evolve}
Suppose at each iteration, trajectory evolution injects $m_{\textnormal{evolve}} \in \{1,\dots,N-1\}$ clean trajectories into the training group alongside $N - m_{\textnormal{evolve}}$ on-policy rollouts. Under Assumption~\ref{asmp:smooth}, for any target $\varepsilon \in (0, p_0)$, the expected redundancy satisfies
\begin{equation}
    E\!\left[\rho(\tau_K^{\textnormal{evolve}})\right] \leq \beta\varepsilon
    \quad \textnormal{for all } K \geq \frac{p_0 - \varepsilon}{\delta_\varepsilon},
    \qquad
    \delta_\varepsilon := h(m_{\textnormal{evolve}})\,\eta\,\Big(1 - (1-\varepsilon)^{\beta T (N - m_{\textnormal{evolve}})}\Big) > 0.
\end{equation}
\end{restatable}

\begin{remark}[Exponential vs.\ Polynomial]
Within this model, the number of iterations TrajEvolve needs to reach redundancy $\beta\varepsilon$ is at most $(p_0 - \varepsilon)/\delta_\varepsilon$, which is polynomial in $(p_0, 1/\varepsilon, 1/\eta, 1/h(m_{\textnormal{evolve}}))$ and, since $\delta_\varepsilon$ is non-decreasing in $T$, bounded independently of the horizon $T$. In contrast, the barrier $K^*$ of Theorem~\ref{thm:main} grows exponentially in $T$. 
\end{remark}
\vspace{-0.4 em}

Due to the space limits, the detailed proof of the above theorem can be found in Appendix~\ref{app:theorem}.

\vspace{-0.8 em}

\section{Experiment}
\vspace{-0.7 em}
\label{sec:experiment}
\subsection{Experiment Setup}
\vspace{-0.5 em}
\textbf{Benchmarks.}
We evaluate our method on two benchmarks: \textit{ALFWorld} \citep{alfworld} is a text-based embodied environment where agents receive natural language goals and must accomplish them through multi-turn interaction. It comprises 3,827 task instances across six categories of household activities: Pick \& Place (Pick), Examine in Light (Look), Clean \& Place (Clean), Heat \& Place (Heat), Cool \& Place (Cool), and Pick Two \& Place (Pick2), requiring compositional reasoning including object identification and state manipulation. \textit{WebShop} \citep{webshop} simulates realistic online shopping scenarios with 1.18 million real-world products and 12,087 crowd-sourced instructions. Agents must interact with a simulated HTML-based e-commerce website to search, navigate through product pages, and purchase items matching user specifications.  This benchmark presents challenges in understanding instructions and strategic exploration across a vast catalog.

\textbf{Baselines.}
We compare our approach against three categories of competitive baselines. (1) \textit{Closed-source LLMs}: GPT-4o \citep{gpt-4o} and Gemini-2.5-Pro \citep{gemini2.5}, which represent frontier models in reasoning capability. (2) \textit{Prompting agents}: ReAct \citep{react} and Reflexion \citep{reflexion}, which leverage in-context prompting and self-reflection mechanisms to guide multi-step decision-making. (3) \textit{RL training methods}: PPO \citep{rl_ppo}, which is a classic actor-critic algorithm that relies on a separate value function for step value prediction; RLOO \citep{rl_rloo} and GRPO \citep{rl_grpo} eliminate the critic by computing relative advantages across multiple sampled rollouts; and GiGPO \citep{rl_gigpo} introduces a two-level structure combining episode-level advantages with step-level advantages via anchor state grouping for fine-grained credit assignment.
\definecolor{cvprblue}{rgb}{0.21,0.49,0.74}

\textbf{Training details.} We use Qwen2.5-1.5B/7B-Instruct as our base models. For fair comparison, all RL training methods share identical hyperparameter configurations, with the rollout group size $N$ set to 8. For the evolution module, we perform trajectory evolution every 5 training epochs. Specifically, for each task, we select 2 trajectories from the rollout samples for subsequent evolution operations, prioritizing trajectories with higher invalid action ratios and longer sequence lengths, as these typically contain more room for improvement. The hyperparameter $\lambda$ is fixed as 0.5. For the discriminator, we use Qwen2.5-1.5B-Instruct as the backbone and pre-train it for 3 epochs before RL optimization. During RL training, the discriminator is fine-tuned via LoRA and updated alternately with the policy network, training for 1 epoch on credit labels derived from trajectories at each iteration. 
\vspace{-0.5em}
\subsection{Overall Evaluation}
\vspace{-0.5 em}
Table \ref{table:performance} shows the experiment results. We summarize our findings as follows:
\vspace{-0.5em}
\begin{itemize}[leftmargin=*]
    \item Closed-source LLMs and prompting-based methods exhibit limited performance on multi-step decision-making tasks. Despite their strong general capabilities, GPT-4o and Gemini-2.5-Pro achieve only 45.3\% and 57.8\% success rates on ALFWorld, respectively. Open-source prompting agents such as ReAct and Reflexion show marginal improvements over vanilla prompting but still underperform significantly, highlighting the inherent difficulty of long-horizon control.
    \item RL training substantially improves agent performance, with group-based methods demonstrating competitive efficiency. PPO brings notable gains across both benchmarks, achieving 64.8\% on ALFWorld and 57.0\% on WebShop with the 1.5B model. GiGPO, with its hierarchical advantage estimation, consistently outperforms other RL baselines, validating the importance of fine-grained credit assignment in multi-step tasks.
    \item Our method achieves the best overall performance across both model scales and benchmarks. On the 1.5B model, our approach attains 92.2\% on ALFWorld and 71.9\% on WebShop, surpassing GiGPO by 5.5\% and 4.7\%, respectively.  These results demonstrate that our trajectory evolution mechanism provides complementary benefits to fine-grained credit assignment, enabling more effective policy optimization for LLM agents.
\end{itemize}
\vspace{-0.9 em}

\begin{table*}[t]
\fontsize{8}{9}\selectfont 
\centering
\caption{Performance Comparison on ALFWorld and Webshop benchmarks. Results are averaged over 3 random seeds. For ALFWorld, we report the success rate (\%) for the results. For WebShop, we report both the average score and the success rate (\%).} 
\label{table:performance}
\vspace{-0.5 em}
\resizebox{\linewidth}{!}{
\renewcommand\arraystretch{1}
\tabcolsep=0.07cm
\scalebox{0.8}{
\begin{tabular}{ccccccccccc}
\toprule
  \multirow{2}{*}{Type} & \multirow{2}{*}{Method} & \multicolumn{7}{c}{ALFWorld}                    & \multicolumn{2}{c}{WebShop} \\ \cmidrule(lr){3-9} \cmidrule(lr){10-11}
                      &                         & Pick & Look & Clean & Heat & Cool & Pick2 & All & Score        & Succ.                \\ \midrule 
\rowcolor{cvprblue!11} \multicolumn{11}{c}{\textit{Closed-Source Model}}   \\
Prompting          & GPT-4o               & 81.3                 & 25.0                 & 36.3                 & 38.5                 & 40.7                 & 23.1                 & 45.3                 & 30.8                 & 18.8                 \\
Prompting          & Gemini-2.5-Pro       & 83.3                 & 42.9                 & 48.6                 & 54.5                 & 64.7                 & 42.9                 & 57.8                 & 42.8                 & 26.6                 \\ \midrule 
\rowcolor{blue!11} \multicolumn{11}{c}{\textit{Qwen2.5-1.5B-Instruct}} \\
Prompting          & Qwen2.5              & 11.1                 & 0.0                  & 7.1                  & 0.0                  & 0.0                  & 7.7                  & 5.5                  & 22.4                 & 5.5                  \\
Prompting          & ReAct                & 17.4                 & 20.5                 & 15.7                 & 6.2                  & 7.7                  & 2.0                  & 13.3                 & 40.1                 & 10.9                 \\
Prompting          & Reflexion            & 35.3                 & 22.2                 & 21.7                 & 13.6                 & 19.4                 & 3.7                  & 21.9                 & 55.8                 & 21.9                 \\
RL Training        & PPO                  & 64.5±2.5             & 72.7±4.8             & 73.1±4.7             & 73.3±5.1             & 41.2±4.0             & 64.3±2.7             & 64.8±3.9             & 75.9±4.1             & 57.0±2.3             \\
RL Training        & RLOO                 & 90.3±3.9             & 54.5±6.0             & 88.5±5.7             & 86.7±6.8             & 47.1±5.5             & 60.7±4.6             & 74.2±3.1             & 80.6±5.1             & 66.7±5.1             \\ 
RL Training        & GRPO                 & 87.1±2.4             & 81.8±4.6             & 84.6±4.4             & 80.0±5.5             & 70.6±4.7             & 64.3±3.8             & 78.1±2.3             & 76.7±3.2             & 58.3±4.3             \\ 
RL Training        & GiGPO                & 94.8±2.5             & \textbf{95.0±4.4}             & 85.8±5.1             & 86.7±6.2             & 82.4±5.1             & 82.6±4.4             & 86.7±2.3             & 80.9±2.1             & 67.2±3.9             \\
RL Training        & Ours                 & \textbf{95.2±2.2} & 87.5±3.2 & \textbf{95.0±3.3} & \textbf{92.9±5.6} & \textbf{91.2±3.7} & \textbf{91.6±2.9} & \textbf{92.2±1.6} & \textbf{83.7±1.4} & \textbf{71.9±1.6}   \\ \midrule

\rowcolor{blue!11} \multicolumn{11}{c}{\textit{Qwen2.5-7B-Instruct}}\\
Prompting          & Qwen2.5              & 31.0                 & 30.3                 & 4.3                  & 10.1                 & 4.4                  & 2.8                  & 14.1                 & 26.0                 & 7.8                  \\
Prompting          & ReAct                & 48.5                 & 35.4                 & 34.3                 & 13.2                 & 18.2                 & 17.6                 & 31.2                 & 46.2                 & 19.5                 \\
Prompting          & Reflexion            & 62.0                 & 41.6                 & 44.9                 & 30.9                 & 36.3                 & 23.8                 & 43.0                 & 58.1                 & 28.9                 \\
RL Training        & PPO                  & 90.3±3.4             & 90.9±3.9             & 84.6±6.9             & 73.3±5.0             & 64.7±5.2             & 71.4±4.2             & 79.7±3.9             & 81.9±2.0               & 67.4±2.3               \\
RL Training        & RLOO                 & 90.3±3.8             & 90.9±4.5             & 88.5±6.0             & 86.7±6.2             & 70.6±4.3             & 64.3±3.8             & 81.5±2.7             & 83.7±2.5             & 76.0±3.9             \\ 
RL Training        & GRPO                 & 96.8±2.7             & 96.4±3.4             & 84.6±7.5             & 86.7±6.8             & 82.4±5.2             & 78.6±4.1             & 87.5±3.1             & 80.8±3.9             & 65.9±4.6             \\ 
RL Training        & GiGPO                & 96.8±2.8             & 96.7±3.0             & 92.3±6.4             & 86.7±5.9             & 88.2±5.1    & 71.4±3.6             & 88.3±2.3             & 82.8±2.5             & 68.8±3.1             \\ 
RL Training        & Ours                 & \textbf{98.5±1.4} & \textbf{98.3±2.9} & \textbf{96.1±4.5} & \textbf{92.9±4.7} & \textbf{89.5±4.4} & \textbf{97.8±3.7} & \textbf{95.3±1.6} & \textbf{90.7±2.2} & \textbf{82.0±2.3}   \\ \bottomrule 
\end{tabular}}}
\end{table*}

\begin{figure*}[t]
    \centering
\setlength{\abovecaptionskip}{3pt}
    \begin{subfigure}[t]{0.245\linewidth}
        \centering
        \includegraphics[width=\linewidth]{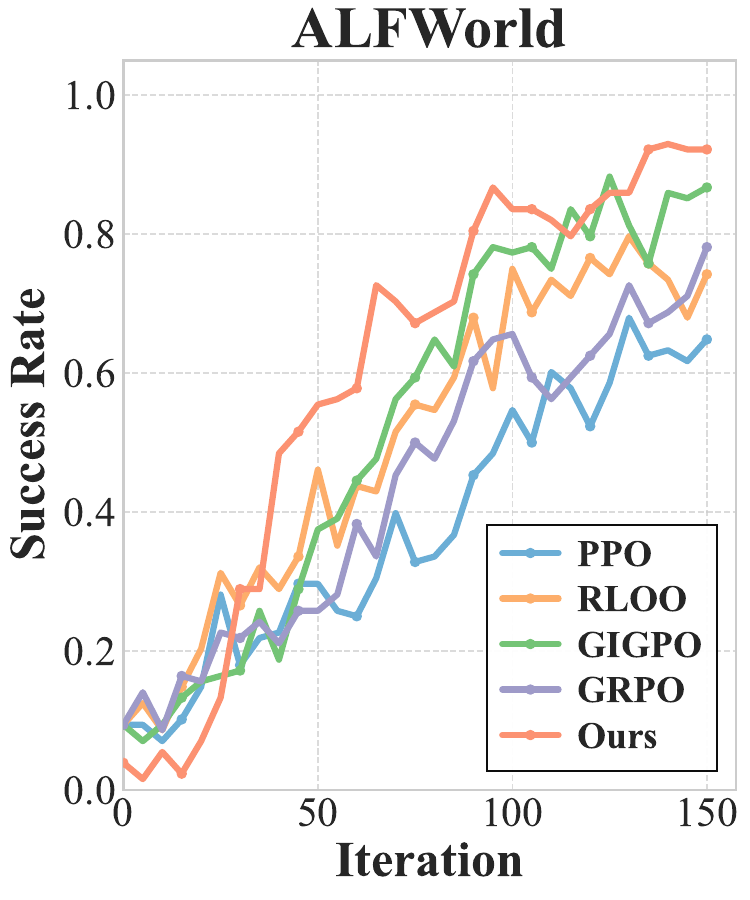}
        \caption{Qwen2.5-1.5B-Instruct}
        \label{fig:alfworld_success_1.5b}
    \end{subfigure}
    \begin{subfigure}[t]{0.245\linewidth}
        \centering
        \includegraphics[width=\linewidth]{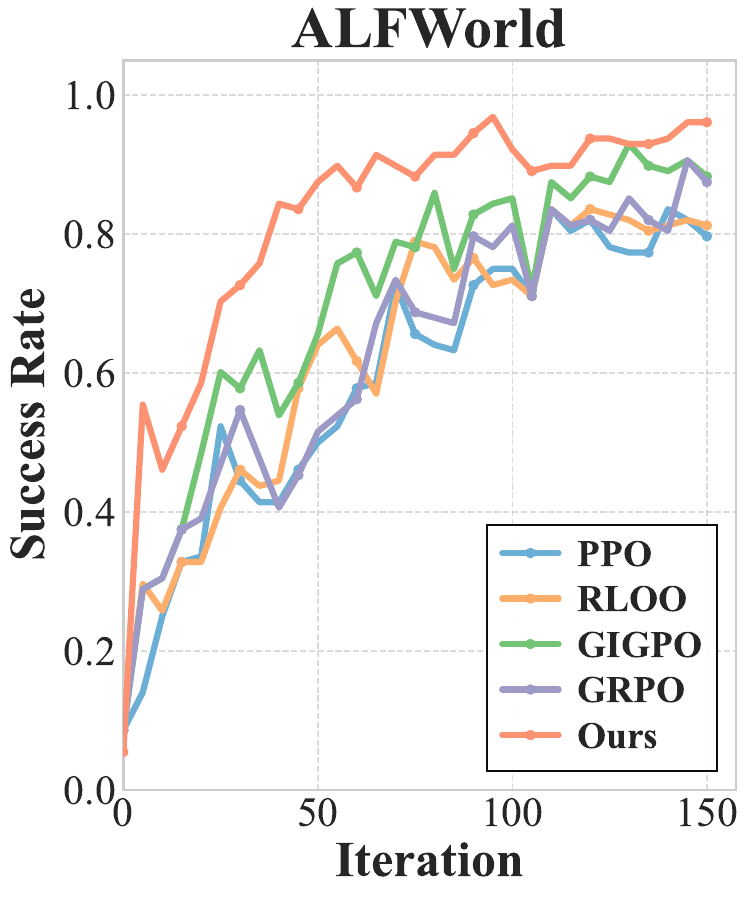}
        \caption{Qwen2.5-7B-Instruct}
        \label{fig:alfworld_success_7b}
    \end{subfigure}
        \begin{subfigure}[t]{0.245\linewidth}
        \centering
        \includegraphics[width=\linewidth]{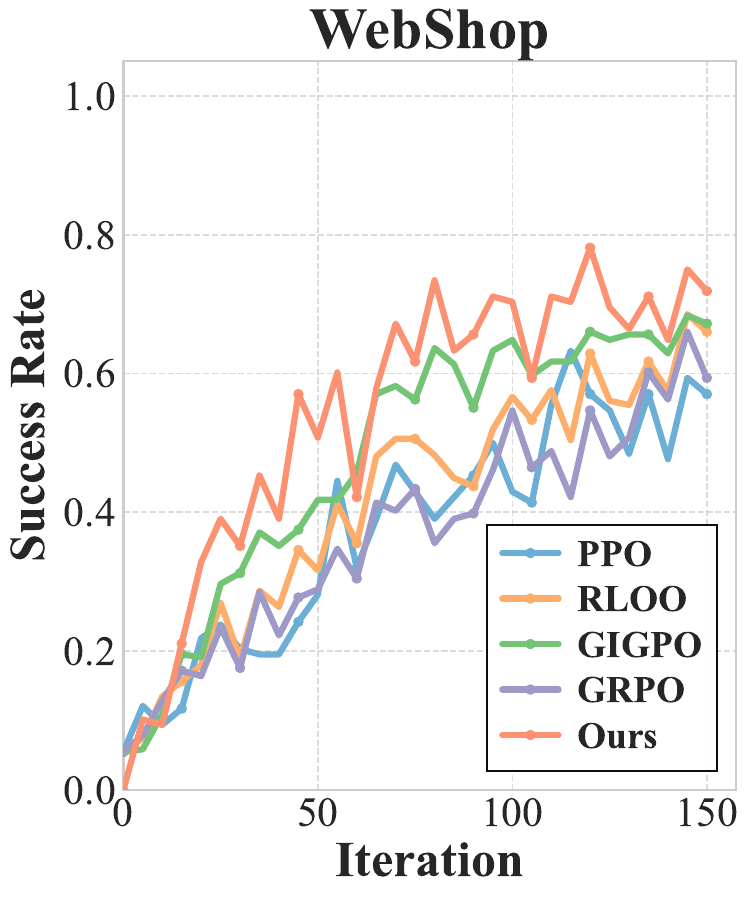}
        \caption{Qwen2.5-1.5B-Instruct}
        \label{fig:webshop_success_1.5b}
    \end{subfigure}
    \begin{subfigure}[t]{0.245\linewidth}
        \centering
        \includegraphics[width=\linewidth]{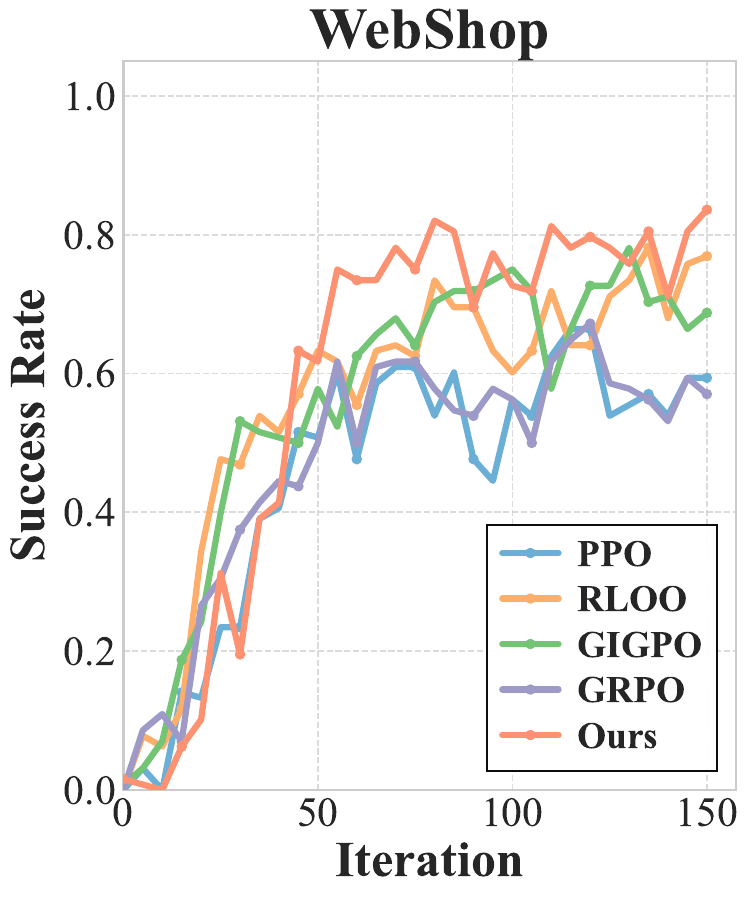}
        \caption{Qwen2.5-7B-Instruct}
        \label{fig:webshop_success_7b}
    \end{subfigure}
    \caption{Validation success rate during training on ALFWorld and WebShop dataset.}
    \vspace{-1em}
    \label{fig:success}
\end{figure*}

\begin{figure*}[t]
    \centering
\setlength{\abovecaptionskip}{3pt}
    \begin{subfigure}[t]{0.245\linewidth}
        \centering
        \includegraphics[width=\linewidth]{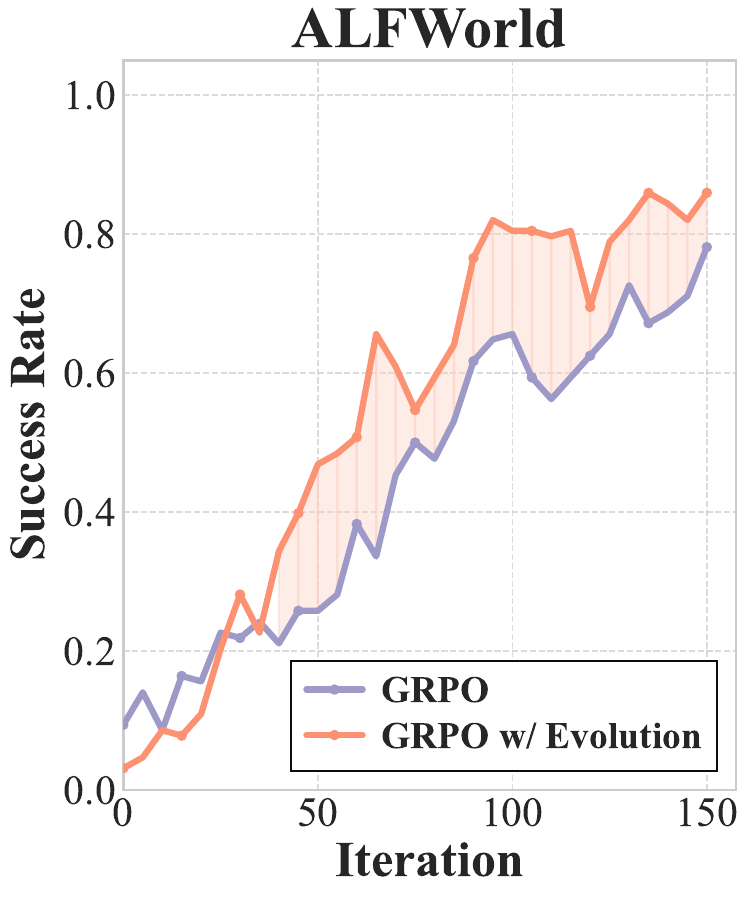}
    \end{subfigure}
    \begin{subfigure}[t]{0.245\linewidth}
        \centering
        \includegraphics[width=\linewidth]{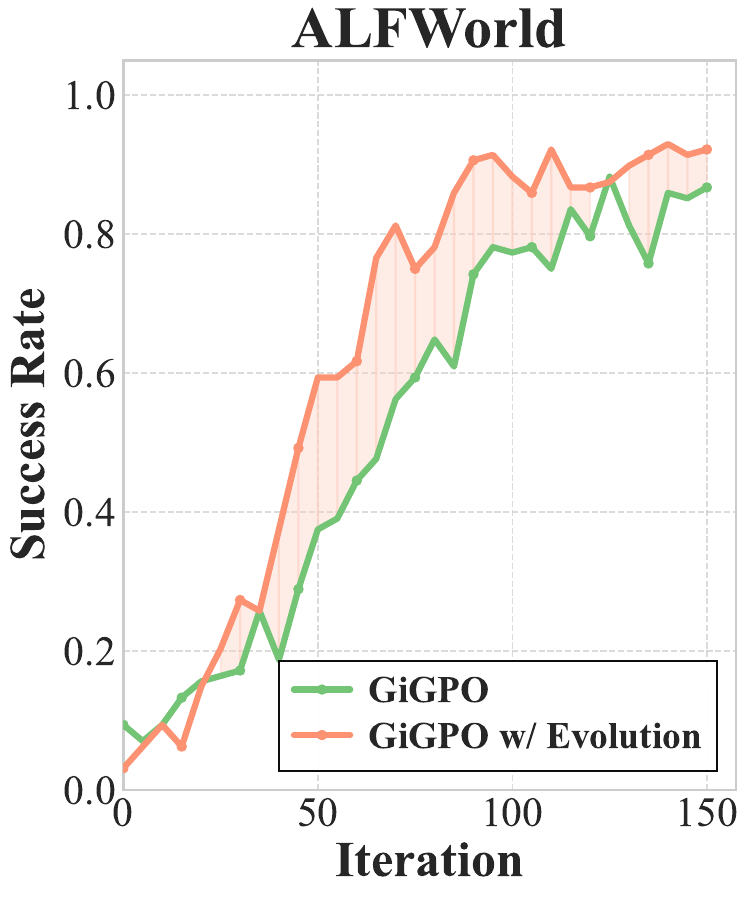}
    \end{subfigure}
        \begin{subfigure}[t]{0.245\linewidth}
        \centering
        \includegraphics[width=\linewidth]{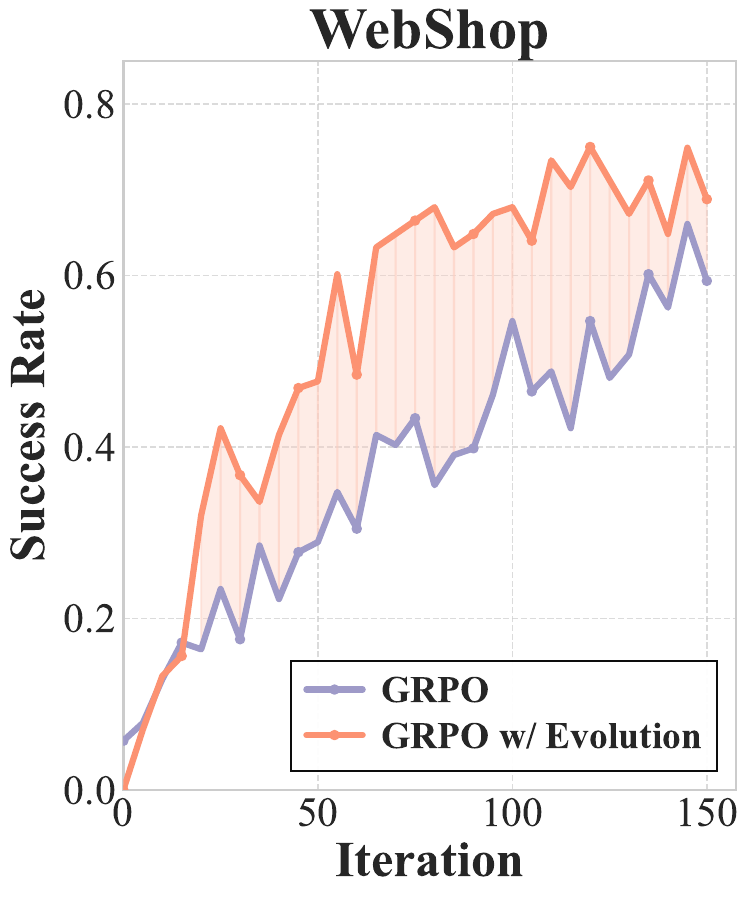}
    \end{subfigure}
    \begin{subfigure}[t]{0.245\linewidth}
        \centering
        \includegraphics[width=\linewidth]{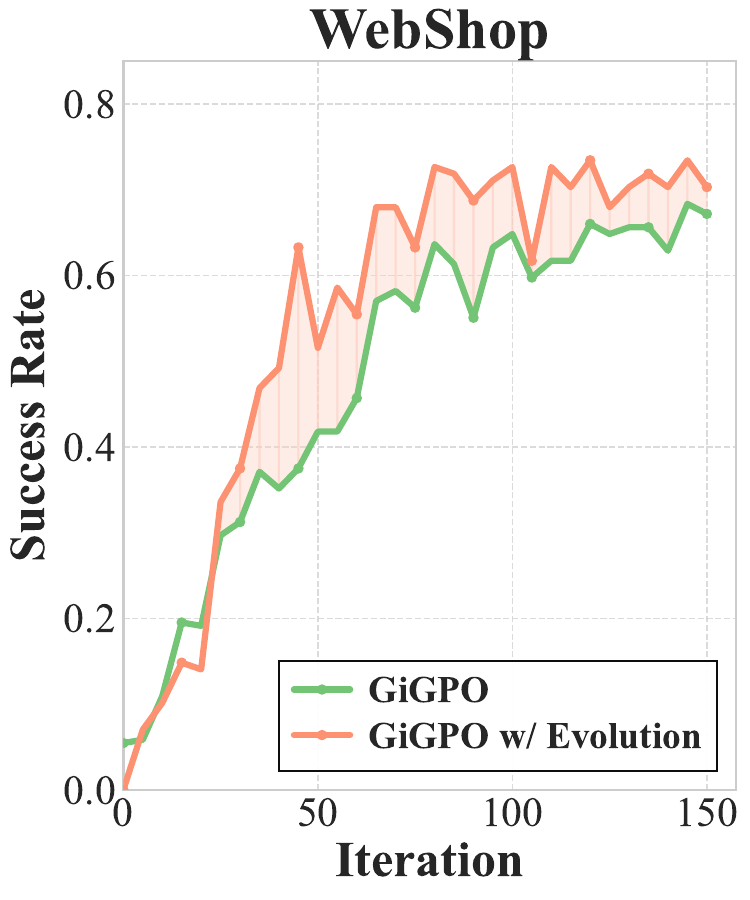}
    \end{subfigure}
    \caption{Effect of integrating the trajectory evolution module with different RL baselines. The shaded area indicates the performance gain from adding our evolution module. }
     \vspace{-0.7 em}
    \label{fig:evolve}
\end{figure*}

\subsection{Training Analysis}
\vspace{-0.5 em}
Figure~\ref{fig:success} presents the validation success rate curves during training, which can show how the training dynamics reflect the effectiveness of our evolution paradigm. We summarize our findings as follows: 
\vspace{-0.5em}
\begin{itemize}[leftmargin=*]
    \item Faster convergence stems from alleviating the sparse signal dilemma. In the early training stages, our method exhibits notably steeper learning curves compared to baselines. This acceleration directly addresses a fundamental limitation of standard RL: the base model struggles to sample sufficient successful trajectories in complex environments, resulting in sparse learning signals. By generating evolutionary trajectories through hindsight refinement, our approach provides additional high-quality training samples that are not bounded by the model's current reasoning capability, thereby enriching the gradient signal and enabling more efficient policy updates.
    \item Sustained performance gains validate the trajectory evolution paradigm. Throughout training, our method consistently maintains a performance advantage over GiGPO and other baselines. Compared to existing RL methods can only improve sampling efficiency of reasoning paths already within the base model's repertoire, our trajectory evolution paradigm actively constructs improved trajectories by eliminating redundant steps. This enables the model to learn from trajectories that are rarely sampled within practical training budgets.
    \item The advantage is more pronounced on WebShop due to higher task complexity. WebShop involves a much larger action space and noisier, lengthier observations, where standard RL methods suffer from indiscriminate updates that encourage redundant steps. Our binary classification approach—determining whether each step is necessary rather than estimating continuous credit scores—compresses the infinite solution space into a verifiable subspace, making credit assignment tractable even in such challenging scenarios. This explains why the performance gap between our method and baselines is particularly evident on WebShop.
\end{itemize}
\vspace{-0.5 em}

\subsection{Compatibility of the Evolution Module}
Figure~\ref{fig:evolve} demonstrates that our trajectory evolution module serves as a general enhancement applicable to different RL baselines. On both WebShop and ALFWorld, integrating the evolution module yields consistent improvements for both GRPO and GiGPO. Notably, the gain is more substantial when combined with GRPO, which lacks fine-grained credit assignment, and the evolved trajectories compensate for this limitation by directly providing refined training samples. For GiGPO, which already performs step-level credit assignment, the evolution module still brings complementary benefits, indicating that actively constructing higher-quality trajectories offers advantages beyond what credit assignment alone can achieve. The consistent patterns across both benchmarks validate the generalizability of our approach.
\vspace{-0.2 em}

\begin{wraptable}{r}{0.58\linewidth}
\fontsize{8}{9}\selectfont 
\centering
\vspace{-1 em}
\caption{Redundancy analysis on ALFWorld and WebShop. ``Avg. Steps'' denotes mean steps across all trajectories. ``Succ. Steps'' measures mean steps for successful trajectories only. Lower values indicate more efficient execution.}
\label{table:length}
\vspace{-1 em}
\resizebox{\linewidth}{!}{
\renewcommand\arraystretch{1}
\tabcolsep=0.07cm
\scalebox{0.8}{
\begin{tabular}{ccccc}
\toprule
\multirow{2}{*}{Model} & \multicolumn{2}{c}{ALFWorld}   & \multicolumn{2}{c}{WebShop}   \\ \cmidrule{2-5} 
                       & Avg.   Steps   & Succ. Steps   & Avg. Steps    & Succ. Steps   \\ \midrule
PPO                    & 17.45          & 10.63         & 8.94          & 8.61          \\ 
RLOO                   & 15.20          & 10.06         & 7.98          & 7.11          \\ 
GRPO                   & 15.36          & 11.26         & 10.97         & 12.61         \\ 
GiGPO                  & 11.58          & 8.76          & 6.16          & 5.33          \\ 
Ours                   & \textbf{9.63}  & \textbf{7.91} & \textbf{5.89} & \textbf{5.15} \\ \bottomrule 
\end{tabular}}}
\end{wraptable}
\subsection{Trajectory Redundancy Analysis}
\vspace{-5 pt}
Existing RL methods typically rely solely on sparse final rewards, which update all steps indiscriminately and may inadvertently encourage redundant actions. To investigate whether our method produces more efficient behaviors, we conduct trajectory length analysis in Table~\ref{table:length}. Since average trajectory length is confounded by success rate, as failed trajectories may be either prolonged by repeated invalid attempts or cut short by premature termination, we additionally report ``Succ. Steps’’, measuring the mean length of successful trajectories only, to enable fair comparison of execution efficiency.
As shown in Table~\ref{table:length}, our method achieves the shortest trajectories on both metrics across all benchmarks. This demonstrates that our binary credit assignment effectively identifies and eliminates redundant steps, enabling the model to complete tasks with more concise action sequences.

\begin{wrapfigure}{r}{0.6\textwidth}
    \centering
    \includegraphics[width=1\linewidth]{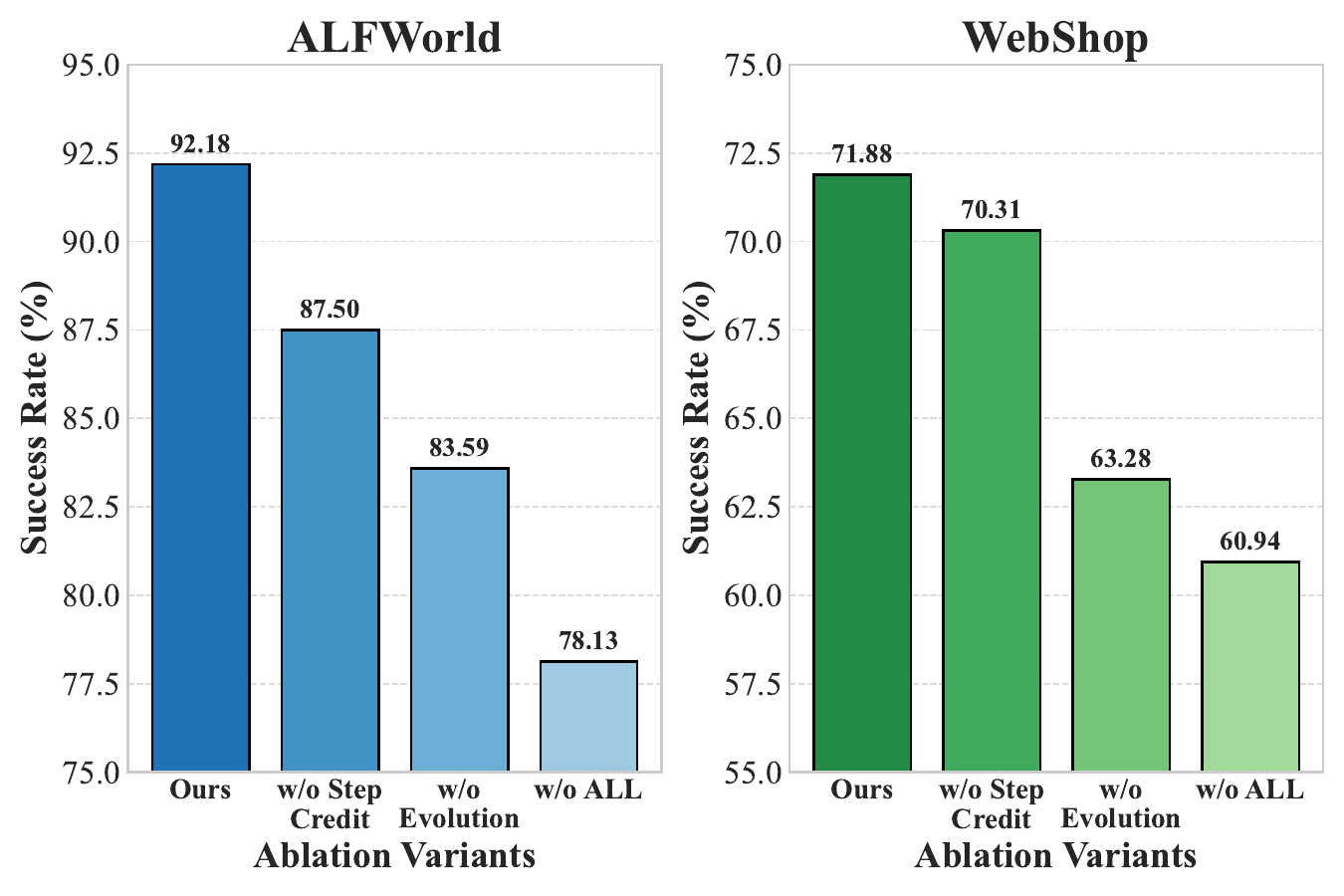}
    \caption{Ablation study on ALFWorld and Webshop benchmarks. Qwen2.5-1.5B-Instruct serves the base models. }
    \vspace{-1em}
    \label{fig:ablation}
\end{wrapfigure}
\vspace{-0.5 em}

\subsection{Ablation Study}
\vspace{-0.5em}
To investigate the individual contribution of each component in our framework, we conduct ablation experiments by systematically removing the step-level credit assignment (w/o Step Credit) and the trajectory evolution module (w/o Evolution). As shown in Figure~\ref{fig:ablation}, both components are essential for achieving optimal performance, with their combination yielding the best results across both benchmarks.
Notably, removing the evolution module causes a larger performance drop compared to removing step credit, indicating that actively constructing higher-quality trajectories provides more substantial benefits than refining credit assignment alone. Furthermore, the gains from both components are largely additive, confirming that trajectory-level optimization and step-level credit assignment address complementary aspects of the challenge.

Due to the space limits, more experiments about complex scenarios generalization, discriminator accuracy, off-policy analysis, and other ablation studies can be referred in Appendix~\ref{app:exp} and ~\ref{app:analysis} .

\vspace{-1 em}

\section{Related Work}
\vspace{-0.5em}
\paragraph{Reinforcement Learning for LLMs.}
Reinforcement learning (RL)~\citep{rl_dpo, song2024trial} has significantly advanced the reasoning capabilities of LLMs. As a representative RL method, PPO~\citep{rl_ppo} achieves desirable performance by enabling dynamic policy updates via environment feedback, yet requires additional critic models for advantage estimation, increasing computational overhead and training complexity. To address this, group-based methods including GRPO~\citep{rl_grpo}, Dr.GRPO~\citep{rl_drgrpo}, DAPO~\citep{rl_dapo}, and GSPO~\citep{rl_gspo} estimate advantages from response batches sampled under the same prompt, eliminating critic dependency while maintaining efficiency. Recent work ~\citep{rl_best} argues that RL primarily improves the sampling efficiency of reasoning paths already within the base model's repertoire, rather than eliciting fundamentally new ones. We take this observation as a starting point rather than a challenge: if the bottleneck of long-horizon RL is sampling efficiency over composed trajectories, then explicitly constructing high-quality compositions from existing rollouts is a principled response. \ourmodel operationalizes this view: every action in an evolved trajectory is drawn from the base policy's own outputs, but the composition is one that on-policy sampling would require exponentially many rollouts to discover (Theorem~\ref{thm:main}).
\vspace{-0.5em}
\paragraph{Credit Assignment in RL.}
Despite the success of group-based RL methods~\citep{rl_rloo}, they rely solely on outcome rewards without evaluating intermediate steps, motivating credit assignment research. Existing methods fall into two categories: i) Heuristic signals~\citep{rl_empg, rl_rlvmr, rl_igpo, hindsight} use LLM confidence or meta-reasoning formats, risking reinforcement of confident but incorrect outputs; ii) Outcome-derived signals~\citep{rl_vineppo, rl_tree-grpo, rl_spo, rl_gigpo} estimate step values via Monte Carlo rollouts but only identify success-leading steps, failing to detect redundant ones. 
\vspace{-0.5em}

\section{Conclusion}
\vspace{-0.5em}
We presented \ourmodel, a step-level credit assignment framework for long-horizon agent tasks. By reformulating credit assignment as binary necessity classification, \ourmodel enables reliable label acquisition through hindsight verification and trains a discriminator to guide both policy optimization and trajectory evolution. Experiments on ALFWorld and WebShop demonstrate that \ourmodel consistently outperforms existing RL methods. Future work includes extending our approach to more complex environments and exploring the integration of step-level credit with traditional search question-answer scenarios.

% \section*{Limitation}

\bibliography{neurips_2026}

@article{gpt-4o,
  title={Gpt-4o system card},
  author={Hurst, Aaron and Lerer, Adam and Goucher, Adam P and Perelman, Adam and Ramesh, Aditya and Clark, Aidan and Ostrow, AJ and Welihinda, Akila and Hayes, Alan and Radford, Alec and others},
  journal={arXiv preprint arXiv:2410.21276},
  year={2024}
}

@article{gemini2.5,
  title={Gemini 2.5: Pushing the frontier with advanced reasoning, multimodality, long context, and next generation agentic capabilities},
  author={Comanici, Gheorghe and Bieber, Eric and Schaekermann, Mike and Pasupat, Ice and Sachdeva, Noveen and Dhillon, Inderjit and Blistein, Marcel and Ram, Ori and Zhang, Dan and Rosen, Evan and others},
  journal={arXiv preprint arXiv:2507.06261},
  year={2025}
}

@article{qwen3,
  title={Qwen3 technical report},
  author={Yang, An and Li, Anfeng and Yang, Baosong and Zhang, Beichen and Hui, Binyuan and Zheng, Bo and Yu, Bowen and Gao, Chang and Huang, Chengen and Lv, Chenxu and others},
  journal={arXiv preprint arXiv:2505.09388},
  year={2025}
}

@article{rl_best,
  title={Does reinforcement learning really incentivize reasoning capacity in llms beyond the base model?},
  author={Yue, Yang and Chen, Zhiqi and Lu, Rui and Zhao, Andrew and Wang, Zhaokai and Song, Shiji and Huang, Gao},
  journal={arXiv preprint arXiv:2504.13837},
  year={2025}
}

@article{rl_dpo,
  title={Direct preference optimization: Your language model is secretly a reward model},
  author={Rafailov, Rafael and Sharma, Archit and Mitchell, Eric and Manning, Christopher D and Ermon, Stefano and Finn, Chelsea},
  journal={Advances in neural information processing systems},
  volume={36},
  pages={53728--53741},
  year={2023}
}

@article{rl_ppo,
  title={Proximal policy optimization algorithms},
  author={Schulman, John and Wolski, Filip and Dhariwal, Prafulla and Radford, Alec and Klimov, Oleg},
  journal={arXiv preprint arXiv:1707.06347},
  year={2017}
}

@article{rl_grpo,
  title={Deepseekmath: Pushing the limits of mathematical reasoning in open language models},
  author={Shao, Zhihong and Wang, Peiyi and Zhu, Qihao and Xu, Runxin and Song, Junxiao and Bi, Xiao and Zhang, Haowei and Zhang, Mingchuan and Li, YK and Wu, Yang and others},
  journal={arXiv preprint arXiv:2402.03300},
  year={2024}
}

@article{rl_gspo,
  title={Group sequence policy optimization},
  author={Zheng, Chujie and Liu, Shixuan and Li, Mingze and Chen, Xiong-Hui and Yu, Bowen and Gao, Chang and Dang, Kai and Liu, Yuqiong and Men, Rui and Yang, An and others},
  journal={arXiv preprint arXiv:2507.18071},
  year={2025}
}

@article{rl_dapo,
  title={Dapo: An open-source llm reinforcement learning system at scale},
  author={Yu, Qiying and Zhang, Zheng and Zhu, Ruofei and Yuan, Yufeng and Zuo, Xiaochen and Yue, Yu and Dai, Weinan and Fan, Tiantian and Liu, Gaohong and Liu, Lingjun and others},
  journal={arXiv preprint arXiv:2503.14476},
  year={2025}
}

@article{rl_drgrpo,
  title={Understanding r1-zero-like training: A critical perspective},
  author={Liu, Zichen and Chen, Changyu and Li, Wenjun and Qi, Penghui and Pang, Tianyu and Du, Chao and Lee, Wee Sun and Lin, Min},
  journal={arXiv preprint arXiv:2503.20783},
  year={2025}
}

@article{rl_gigpo,
  title={Group-in-group policy optimization for llm agent training},
  author={Feng, Lang and Xue, Zhenghai and Liu, Tingcong and An, Bo},
  journal={arXiv preprint arXiv:2505.10978},
  year={2025}
}

@article{rl_rlvmr,
  title={Rlvmr: Reinforcement learning with verifiable meta-reasoning rewards for robust long-horizon agents},
  author={Zhang, Zijing and Chen, Ziyang and Li, Mingxiao and Tu, Zhaopeng and Li, Xiaolong},
  journal={arXiv preprint arXiv:2507.22844},
  year={2025}
}

@article{rl_tree-grpo,
  title={Tree search for llm agent reinforcement learning},
  author={Ji, Yuxiang and Ma, Ziyu and Wang, Yong and Chen, Guanhua and Chu, Xiangxiang and Wu, Liaoni},
  journal={arXiv preprint arXiv:2509.21240},
  year={2025}
}

@article{rl_igpo,
  title={Information Gain-based Policy Optimization: A Simple and Effective Approach for Multi-Turn LLM Agents},
  author={Wang, Guoqing and Dai, Sunhao and Ye, Guangze and Gan, Zeyu and Yao, Wei and Deng, Yong and Wu, Xiaofeng and Ying, Zhenzhe},
  journal={arXiv preprint arXiv:2510.14967},
  year={2025}
}

@article{rl_gepo,
  title={Graph-Enhanced Policy Optimization in LLM Agent Training},
  author={Yuan, Jiazhen and Zhao, Wei and Bai, Zhengbiao},
  journal={arXiv preprint arXiv:2510.26270},
  year={2025}
}

@inproceedings{rl_rloo,
  title={Back to Basics: Revisiting REINFORCE-Style Optimization for Learning from Human Feedback in LLMs},
  author={Ahmadian, Arash and Cremer, Chris and Gall{\'e}, Matthias and Fadaee, Marzieh and Kreutzer, Julia and Pietquin, Olivier and {\"U}st{\"u}n, Ahmet and Hooker, Sara},
  booktitle={Proceedings of the 62nd Annual Meeting of the Association for Computational Linguistics (Volume 1: Long Papers)},
  pages={12248--12267},
  year={2024}
}

@article{rl_spo,
  title={Segment policy optimization: Effective segment-level credit assignment in rl for large language models},
  author={Guo, Yiran and Xu, Lijie and Liu, Jie and Ye, Dan and Qiu, Shuang},
  journal={arXiv preprint arXiv:2505.23564},
  year={2025}
}

@inproceedings{rl_vineppo,
  title={VinePPO: Refining Credit Assignment in RL Training of LLMs},
  author={Kazemnejad, Amirhossein and Aghajohari, Milad and Portelance, Eva and Sordoni, Alessandro and Reddy, Siva and Courville, Aaron and Le Roux, Nicolas},
  booktitle={Forty-second International Conference on Machine Learning}
}

@article{rl_off_policy_GRPO,
  title={Revisiting Group Relative Policy Optimization: Insights into On-Policy and Off-Policy Training},
  author={Mroueh, Youssef and Dupuis, Nicolas and Belgodere, Brian and Nitsure, Apoorva and Rigotti, Mattia and Greenewald, Kristjan and Navratil, Jiri and Ross, Jerret and Rios, Jesus},
  journal={arXiv preprint arXiv:2505.22257},
  year={2025}
}

@article{rl_arpo,
  title={ARPO: End-to-End Policy Optimization for GUI Agents with Experience Replay},
  author={Lu, Fanbin and Zhong, Zhisheng and Liu, Shu and Fu, Chi-Wing and Jia, Jiaya},
  journal={arXiv preprint arXiv:2505.16282},
  year={2025}
}

@article{rl_exgrpo,
  title={ExGRPO: Learning to Reason from Experience},
  author={Zhan, Runzhe and Li, Yafu and Wang, Zhi and Qu, Xiaoye and Liu, Dongrui and Shao, Jing and Wong, Derek F and Cheng, Yu},
  journal={arXiv preprint arXiv:2510.02245},
  year={2025}
}

@article{alfworld,
  title={Alfworld: Aligning text and embodied environments for interactive learning},
  author={Shridhar, Mohit and Yuan, Xingdi and C{\^o}t{\'e}, Marc-Alexandre and Bisk, Yonatan and Trischler, Adam and Hausknecht, Matthew},
  journal={arXiv preprint arXiv:2010.03768},
  year={2020}
}

@article{webshop,
  title={Webshop: Towards scalable real-world web interaction with grounded language agents},
  author={Yao, Shunyu and Chen, Howard and Yang, John and Narasimhan, Karthik},
  journal={Advances in Neural Information Processing Systems},
  volume={35},
  pages={20744--20757},
  year={2022}
}

@inproceedings{react,
  title={React: Synergizing reasoning and acting in language models},
  author={Yao, Shunyu and Zhao, Jeffrey and Yu, Dian and Du, Nan and Shafran, Izhak and Narasimhan, Karthik R and Cao, Yuan},
  booktitle={The eleventh international conference on learning representations},
  year={2022}
}

@article{reflexion,
  title={Reflexion: Language agents with verbal reinforcement learning},
  author={Shinn, Noah and Cassano, Federico and Gopinath, Ashwin and Narasimhan, Karthik and Yao, Shunyu},
  journal={Advances in Neural Information Processing Systems},
  volume={36},
  pages={8634--8652},
  year={2023}
}

@article{gpt4,
  title={Gpt-4 technical report},
  author={Achiam, Josh and Adler, Steven and Agarwal, Sandhini and Ahmad, Lama and Akkaya, Ilge and Aleman, Florencia Leoni and Almeida, Diogo and Altenschmidt, Janko and Altman, Sam and Anadkat, Shyamal and others},
  journal={arXiv preprint arXiv:2303.08774},
  year={2023}
}

@article{rl_istar,
  title={Agentic Reinforcement Learning with Implicit Step Rewards},
  author={Liu, Xiaoqian and Wang, Ke and Wu, Yuchuan and Huang, Fei and Li, Yongbin and Zhang, Junge and Jiao, Jianbin},
  journal={arXiv preprint arXiv:2509.19199},
  year={2025}
}

@article{song2024trial,
  title={Trial and error: Exploration-based trajectory optimization for llm agents},
  author={Song, Yifan and Yin, Da and Yue, Xiang and Huang, Jie and Li, Sujian and Lin, Bill Yuchen},
  journal={arXiv preprint arXiv:2403.02502},
  year={2024}
}

@article{ouyang2022training,
  title={Training language models to follow instructions with human feedback},
  author={Ouyang, Long and Wu, Jeffrey and Jiang, Xu and Almeida, Diogo and Wainwright, Carroll and Mishkin, Pamela and Zhang, Chong and Agarwal, Sandhini and Slama, Katarina and Ray, Alex and others},
  journal={Advances in neural information processing systems},
  volume={35},
  pages={27730--27744},
  year={2022}
}

@article{bai2022training,
  title={Training a helpful and harmless assistant with reinforcement learning from human feedback},
  author={Bai, Yuntao and Jones, Andy and Ndousse, Kamal and Askell, Amanda and Chen, Anna and DasSarma, Nova and Drain, Dawn and Fort, Stanislav and Ganguli, Deep and Henighan, Tom and others},
  journal={arXiv preprint arXiv:2204.05862},
  year={2022}
}

@article{guo2025deepseek,
  title={Deepseek-r1: Incentivizing reasoning capability in llms via reinforcement learning},
  author={Guo, Daya and Yang, Dejian and Zhang, Haowei and Song, Junxiao and Wang, Peiyi and Zhu, Qihao and Xu, Runxin and Zhang, Ruoyu and Ma, Shirong and Bi, Xiao and others},
  journal={arXiv preprint arXiv:2501.12948},
  year={2025}
}

@article{yang2024qwen2,
  title={Qwen2. 5-math technical report: Toward mathematical expert model via self-improvement},
  author={Yang, An and Zhang, Beichen and Hui, Binyuan and Gao, Bofei and Yu, Bowen and Li, Chengpeng and Liu, Dayiheng and Tu, Jianhong and Zhou, Jingren and Lin, Junyang and others},
  journal={arXiv preprint arXiv:2409.12122},
  year={2024}
}

@article{team2025kimi,
  title={Kimi k1. 5: Scaling reinforcement learning with llms},
  author={Team, Kimi and Du, Angang and Gao, Bofei and Xing, Bowei and Jiang, Changjiu and Chen, Cheng and Li, Cheng and Xiao, Chenjun and Du, Chenzhuang and Liao, Chonghua and others},
  journal={arXiv preprint arXiv:2501.12599},
  year={2025}
}

@article{yue2025vapo,
  title={Vapo: Efficient and reliable reinforcement learning for advanced reasoning tasks},
  author={Yue, Yu and Yuan, Yufeng and Yu, Qiying and Zuo, Xiaochen and Zhu, Ruofei and Xu, Wenyuan and Chen, Jiaze and Wang, Chengyi and Fan, TianTian and Du, Zhengyin and others},
  journal={arXiv preprint arXiv:2504.05118},
  year={2025}
}

@article{li2023remax,
  title={Remax: A simple, effective, and efficient reinforcement learning method for aligning large language models},
  author={Li, Ziniu and Xu, Tian and Zhang, Yushun and Lin, Zhihang and Yu, Yang and Sun, Ruoyu and Luo, Zhi-Quan},
  journal={arXiv preprint arXiv:2310.10505},
  year={2023}
}

@article{gulcehre2023reinforced,
  title={Reinforced self-training (rest) for language modeling},
  author={Gulcehre, Caglar and Paine, Tom Le and Srinivasan, Srivatsan and Konyushkova, Ksenia and Weerts, Lotte and Sharma, Abhishek and Siddhant, Aditya and Ahern, Alex and Wang, Miaosen and Gu, Chenjie and others},
  journal={arXiv preprint arXiv:2308.08998},
  year={2023}
}

@article{singh2023beyond,
  title={Beyond human data: Scaling self-training for problem-solving with language models},
  author={Singh, Avi and Co-Reyes, John D and Agarwal, Rishabh and Anand, Ankesh and Patil, Piyush and Garcia, Xavier and Liu, Peter J and Harrison, James and Lee, Jaehoon and Xu, Kelvin and others},
  journal={arXiv preprint arXiv:2312.06585},
  year={2023}
}

@article{dong2023raft,
  title={Raft: Reward ranked finetuning for generative foundation model alignment},
  author={Dong, Hanze and Xiong, Wei and Goyal, Deepanshu and Zhang, Yihan and Chow, Winnie and Pan, Rui and Diao, Shizhe and Zhang, Jipeng and Shum, Kashun and Zhang, Tong},
  journal={arXiv preprint arXiv:2304.06767},
  year={2023}
}

@article{yuan2024self,
  title={Self-rewarding language models},
  author={Yuan, Weizhe and Pang, Richard Yuanzhe and Cho, Kyunghyun and Li, Xian and Sukhbaatar, Sainbayar and Xu, Jing and Weston, Jason},
  journal={arXiv preprint arXiv:2401.10020},
  year={2024}
}

@article{chen2024self,
  title={Self-play fine-tuning converts weak language models to strong language models},
  author={Chen, Zixiang and Deng, Yihe and Yuan, Huizhuo and Ji, Kaixuan and Gu, Quanquan},
  journal={arXiv preprint arXiv:2401.01335},
  year={2024}
}

@article{roux2025tapered,
  title={Tapered off-policy reinforce: Stable and efficient reinforcement learning for llms},
  author={Roux, Nicolas Le and Bellemare, Marc G and Lebensold, Jonathan and Bergeron, Arnaud and Greaves, Joshua and Fr{\'e}chette, Alex and Pelletier, Carolyne and Thibodeau-Laufer, Eric and Toth, S{\'a}ndor and Work, Sam},
  journal={arXiv preprint arXiv:2503.14286},
  year={2025}
}

@article{yan2025learning,
  title={Learning to reason under off-policy guidance},
  author={Yan, Jianhao and Li, Yafu and Hu, Zican and Wang, Zhi and Cui, Ganqu and Qu, Xiaoye and Cheng, Yu and Zhang, Yue},
  journal={arXiv preprint arXiv:2504.14945},
  year={2025}
}

@article{bartoldson2025trajectory,
  title={Trajectory balance with asynchrony: Decoupling exploration and learning for fast, scalable llm post-training},
  author={Bartoldson, Brian and Venkatraman, Siddarth and Diffenderfer, James and Jain, Moksh and Ben-Nun, Tal and Lee, Seanie and Kim, Minsu and Obando-Ceron, Johan and Bengio, Yoshua and Kailkhura, Bhavya},
  journal={arXiv preprint arXiv:2503.18929},
  year={2025}
}

@article{zhou2024archer,
  title={Archer: Training language model agents via hierarchical multi-turn rl},
  author={Zhou, Yifei and Zanette, Andrea and Pan, Jiayi and Levine, Sergey and Kumar, Aviral},
  journal={arXiv preprint arXiv:2402.19446},
  year={2024}
}

@inproceedings{xi2025agentgym,
  title={Agentgym: Evaluating and training large language model-based agents across diverse environments},
  author={Xi, Zhiheng and Ding, Yiwen and Chen, Wenxiang and Hong, Boyang and Guo, Honglin and Wang, Junzhe and Guo, Xin and Yang, Dingwen and Liao, Chenyang and He, Wei and others},
  booktitle={Proceedings of the 63rd Annual Meeting of the Association for Computational Linguistics (Volume 1: Long Papers)},
  pages={27914--27961},
  year={2025}
}

@article{hu2025training,
  title={Training Task Reasoning LLM Agents for Multi-turn Task Planning via Single-turn Reinforcement Learning},
  author={Hu, Hanjiang and Liu, Changliu and Li, Na and Wang, Yebin},
  journal={arXiv preprint arXiv:2509.20616},
  year={2025}
}

@article{rl_rest,
  title={ResT: Reshaping Token-Level Policy Gradients for Tool-Use Large Language Models},
  author={Lin, Zihan and Wang, Xiaohan and Cao, Jie and Chai, Jiajun and Yin, Guojun and Lin, Wei and He, Ran},
  journal={arXiv preprint arXiv:2509.21826},
  year={2025}
}

@article{rl_portool,
  title={PORTool: Tool-Use LLM Training with Rewarded Tree},
  author={Wu, Feijie and Zhu, Weiwu and Zhang, Yuxiang and Chatterjee, Soumya and Zhu, Jiarong and Mo, Fan and Luo, Rodin and Gao, Jing},
  journal={arXiv preprint arXiv:2510.26020},
  year={2025}
}

@article{criticsearch,
  title={CriticSearch: Fine-Grained Credit Assignment for Search Agents via a Retrospective Critic},
  author={Zhang, Yaocheng and Huang, Haohuan and Song, Zijun and Zhu, Yuanheng and Zhang, Qichao and Zhao, Zijie and Zhao, Dongbin},
  journal={arXiv preprint arXiv:2511.12159},
  year={2025}
}

@article{prime,
  title={Process reinforcement through implicit rewards},
  author={Cui, Ganqu and Yuan, Lifan and Wang, Zefan and Wang, Hanbin and Zhang, Yuchen and Chen, Jiacheng and Li, Wendi and He, Bingxiang and Fan, Yuchen and Yu, Tianyu and others},
  journal={arXiv preprint arXiv:2502.01456},
  year={2025}
}

@article{pure,
  title={Stop summation: Min-form credit assignment is all process reward model needs for reasoning},
  author={Cheng, Jie and Qiao, Ruixi and Li, Lijun and Guo, Chao and Wang, Junle and Xiong, Gang and Lv, Yisheng and Wang, Fei-Yue},
  journal={arXiv preprint arXiv:2504.15275},
  year={2025}
}

@article{stepsearch,
  title={StepSearch: Igniting LLMs Search Ability via Step-Wise Proximal Policy Optimization},
  author={Wang, Ziliang and Zheng, Xuhui and An, Kang and Ouyang, Cijun and Cai, Jialu and Wang, Yuhang and Wu, Yichao},
  journal={arXiv preprint arXiv:2505.15107},
  year={2025}
}

@article{hindsight,
  title={Hindsight credit assignment},
  author={Harutyunyan, Anna and Dabney, Will and Mesnard, Thomas and Gheshlaghi Azar, Mohammad and Piot, Bilal and Heess, Nicolas and van Hasselt, Hado P and Wayne, Gregory and Singh, Satinder and Precup, Doina and others},
  journal={Advances in neural information processing systems},
  volume={32},
  year={2019}
}

@inproceedings{lare,
  title={Latent reward: Llm-empowered credit assignment in episodic reinforcement learning},
  author={Qu, Yun and Jiang, Yuhang and Wang, Boyuan and Mao, Yixiu and Wang, Cheems and Liu, Chang and Ji, Xiangyang},
  booktitle={Proceedings of the AAAI Conference on Artificial Intelligence},
  volume={39},
  number={19},
  pages={20095--20103},
  year={2025}
}

@article{Implicitprm,
  title={Free process rewards without process labels},
  author={Yuan, Lifan and Li, Wendi and Chen, Huayu and Cui, Ganqu and Ding, Ning and Zhang, Kaiyan and Zhou, Bowen and Liu, Zhiyuan and Peng, Hao},
  journal={arXiv preprint arXiv:2412.01981},
  year={2024}
}

@article{acpo,
  title={Pinpointing crucial steps: Attribution-based Credit Assignment for Verifiable Reinforcement Learning},
  author={Yin, Junxi and Luo, Haisen and Li, Zhenyu and Liu, Yihua and Liu, Dan and Li, Zequn and Xu, Xiaohang},
  journal={arXiv preprint arXiv:2510.08899},
  year={2025}
}

@article{carl,
  title={CARL: Critical Action Focused Reinforcement Learning for Multi-Step Agent},
  author={Shen, Leyang and Zhang, Yang and Ling, Chun Kai and Zhao, Xiaoyan and Chua, Tat-Seng},
  journal={arXiv preprint arXiv:2512.04949},
  year={2025}
}

@article{spa-rl,
  title={SPA-RL: Reinforcing LLM Agents via Stepwise Progress Attribution},
  author={Wang, Hanlin and Leong, Chak Tou and Wang, Jiashuo and Wang, Jian and Li, Wenjie},
  journal={arXiv preprint arXiv:2505.20732},
  year={2025}
}

@article{rl_rlver,
  title={Rlver: Reinforcement learning with verifiable emotion rewards for empathetic agents},
  author={Wang, Peisong and Ma, Ruotian and Zhang, Bang and Chen, Xingyu and He, Zhiwei and Luo, Kang and Lv, Qingsong and Jiang, Qingxuan and Xie, Zheng and Wang, Shanyi and others},
  journal={arXiv preprint arXiv:2507.03112},
  year={2025}
}

@article{rl_empg,
  title={Harnessing uncertainty: Entropy-modulated policy gradients for long-horizon llm agents},
  author={Wang, Jiawei and Liu, Jiacai and Fu, Yuqian and Li, Yingru and Wang, Xintao and Lin, Yuan and Yue, Yu and Zhang, Lin and Wang, Yang and Wang, Ke},
  journal={arXiv preprint arXiv:2509.09265},
  year={2025}
}

@article{hz_hindsight,
  title={Hindsight credit assignment for long-horizon llm agents},
  author={Tan, Hui-Ze and Yang, Xiao-Wen and Chen, Hao and Shao, Jie-Jing and Wen, Yi and Shen, Yuteng and Luo, Weihong and Du, Xiku and Guo, Lan-Zhe and Li, Yu-Feng},
  journal={arXiv preprint arXiv:2603.08754},
  year={2026}
}

@misc{qr_gca,
      title={Grounded Credit Assignment in Reinforcement Learning for Multi-Turn Search Agents}, 
      author={Qirui Liu and Hao Chen and Yi Wen and Weijie Shi and Yao Zhao and Xianhe Chen and Hao Zhao and Xing Tang and Weihong Luo and Yuyu Luo},
      year={2026},
      eprint={2610.23148},
      archivePrefix={arXiv},
      primaryClass={cs.LG},
      url={https://arxiv.org/abs/2610.23148}, 
}

@misc{yx_finhopper,
      title={FinHopper: A Progressive Multi-Hop Benchmark for Complex Financial Agents with Step-wise Evaluation}, 
      author={Yuxuan Wu and Miaopei Lin and Qiufang Ying and Lin Dan and Hao Zhao and Qirui Liu and Yuteng Shen and Shaohui Wu and Weihong Luo and Xiku Du and Xing Tang and Hao Chen},
      year={2026},
      eprint={2610.26036},
      archivePrefix={arXiv},
      primaryClass={cs.LG},
      url={https://arxiv.org/abs/2610.26036}, 
}

@misc{zs_r2-searcher,
      title={R$^2$-Searcher: Calibrating Retrieval and Reasoning Boundaries for Agentic Search}, 
      author={Sheng Zhang and Junyi Li and Wenlin Zhang and Xiaowei Qian and Yichao Wang and Yingyi Zhang and Maolin Wang and Yong Liu and Xiangyu Zhao},
      year={2026},
      eprint={2606.28566},
      archivePrefix={arXiv},
      primaryClass={cs.IR},
      url={https://arxiv.org/abs/2606.28566}, 
}

@inproceedings{xdr_harnessing,
  title={Harnessing large language models for knowledge graph question answering via adaptive multi-aspect retrieval-augmentation},
  author={Xu, Derong and Li, Xinhang and Zhang, Ziheng and Lin, Zhenxi and Zhu, Zhihong and Zheng, Zhi and Wu, Xian and Zhao, Xiangyu and Xu, Tong and Chen, Enhong},
  booktitle={Proceedings of the AAAI Conference on Artificial Intelligence},
  volume={39},
  number={24},
  pages={25570--25578},
  year={2025}
}
\bibliographystyle{plainnat}

%%%%%%%%%%%%%%%%%%%%%%%%%%%%%%%%%%%%%%%%%%%%%%%%%%%%%%%%%%%%

\appendix
\newpage
\appendix
\onecolumn
\begin{center}
    \LARGE \textbf{Technical Appendix}
    \vspace{1em}
\end{center}
\addcontentsline{toc}{part}{Appendix}  %

\begingroup
  \etocsetnexttocdepth{subsection}     %
  \etocsettocstyle{\section*{Table of Contents}}{} %
  \localtableofcontents                %
\endgroup
\clearpage

\section{Proof of Theorem~\ref{thm:main}}
\label{app:theorem}

\redundancyasmp*

\smoothness*

\begin{lemma}[Sampling Barrier]
\label{lem:sampling}
For any policy $\pi_\theta$ satisfying $P_{\pi_\theta}(r_t = 1 \mid s_t) \geq p$ at a $\beta$-fraction of the states along every rollout, the probability that a trajectory of length $T$ is clean satisfies:
\begin{equation}
    P_{\pi_\theta}(\tau_{\text{clean}}) \leq (1 - p)^{\beta T}.
\end{equation}
In a group of $N$ independent rollouts, the number of clean trajectories $m_{\textnormal{clean}}$ satisfies $E[m_{\textnormal{clean}}] \leq N (1 - p)^{\beta T}$, $P(m_{\textnormal{clean}} \geq 1) \leq N (1 - p)^{\beta T}$, and $P(m_{\textnormal{clean}} = N) \leq (1 - p)^{\beta T N}$.
\end{lemma}

\begin{proof}
A clean trajectory requires, in particular, each of the (at least) $\beta T$ redundancy-prone steps to be non-redundant. Conditional on the past, each such step is non-redundant with probability at most $1 - p$, so by the chain rule $P(\text{clean}) \leq (1 - p)^{\beta T}$. Linearity of expectation gives the bound on $E[m_{\textnormal{clean}}]$, Markov's inequality gives $P(m_{\textnormal{clean}} \geq 1) \leq E[m_{\textnormal{clean}}]$, and independence across rollouts gives the bound on $P(m_{\textnormal{clean}} = N)$.
\end{proof}

\begin{lemma}[Self-Locking of the Redundancy Floor]
\label{lem:locking}
Under Assumptions~\ref{asmp:redundancy} and~\ref{asmp:smooth}, define the critical iteration count:
\begin{equation}
    K^* \triangleq \frac{p_0}{2 \, N h_{\max} \, \eta \, (1 - p_0/2)^{\beta T}}.
\end{equation}
For all $K \leq K^*$, the per-step redundancy probability remains above half its initial value throughout training:
\begin{equation}
    p_k \geq \frac{p_0}{2}, \quad \forall\, k \leq K.
\end{equation}
\end{lemma}

\begin{proof}
By contradiction. Suppose $k^*$ is the first iteration at which $p_{k^*} < p_0/2$. For all $k < k^*$, we have $p_k \geq p_0/2$, so by Lemma~\ref{lem:sampling} the probability that the group contains at least one clean trajectory satisfies $P(m_{\textnormal{clean}}^{(k)} \geq 1) \leq N (1 - p_0/2)^{\beta T}$. By Assumption~\ref{asmp:smooth}, using $h(0) = 0$ and $h \leq h_{\max}$:
\begin{equation}
    E[\Delta p_k] = \sum_{m=0}^{N} P\big(m_{\textnormal{clean}}^{(k)} = m\big)\, h(m)\, \eta
    \leq h_{\max}\, \eta \, P\big(m_{\textnormal{clean}}^{(k)} \geq 1\big)
    \leq N h_{\max}\, \eta \, (1 - p_0/2)^{\beta T}.
\end{equation}
The cumulative reduction over $k^*$ iterations satisfies:
\begin{equation}
    p_0 - p_{k^*} \leq k^* \cdot N h_{\max}\, \eta \, (1 - p_0/2)^{\beta T}.
\end{equation}
For $p_{k^*} < p_0/2$, the left-hand side exceeds $p_0/2$, requiring $k^* > K^*$. This contradicts $k^* \leq K \leq K^*$.
\end{proof}

\begin{remark}[Self-Locking Interpretation]
\label{rmk:locking}
Lemma~\ref{lem:locking} reveals a self-reinforcing mechanism: the redundancy probability $p_k$ decreases slowly \emph{because} $p_k$ is large enough to make clean trajectories exponentially rare, which in turn limits the useful gradient signal for reducing $p_k$. On-policy methods cannot escape this cycle.
\end{remark}

\redundancy*

\begin{proof}
By Lemma~\ref{lem:locking}, $p_K \geq p_0/2$ for all $K \leq K^*$. By Assumption~\ref{asmp:redundancy}, each of the $\beta T$ redundancy-prone steps of $\tau_K$ is redundant with probability $p_K$ and the remaining steps are never redundant, so
\begin{equation}
    E[\rho(\tau_K)] = \frac{\beta T \cdot p_K}{T} = \beta p_K \geq \frac{\beta p_0}{2}.
\end{equation}
\end{proof}

\trajevolve*

\begin{proof}
Since the group contains $m_{\textnormal{evolve}} \geq 1$ injected clean trajectories, $m_{\textnormal{clean}}^{(k)} \geq m_{\textnormal{evolve}}$ at every iteration, so the case $m_{\textnormal{clean}}^{(k)} = 0$ never occurs; moreover, since $h \geq 0$, $p_k$ is non-increasing in expectation.

Fix $\varepsilon \in (0, p_0)$ and consider an iteration $k$ with $p_k \geq \varepsilon$. By Assumption~\ref{asmp:smooth}, no progress is made only when $m_{\textnormal{clean}}^{(k)} = N$ (as $h(N) = 0$), i.e., only when all $N - m_{\textnormal{evolve}}$ on-policy rollouts are clean simultaneously, which by Lemma~\ref{lem:sampling} has probability at most $(1 - p_k)^{\beta T (N - m_{\textnormal{evolve}})} \leq (1 - \varepsilon)^{\beta T (N - m_{\textnormal{evolve}})}$. Otherwise, $m_{\textnormal{evolve}} \leq m_{\textnormal{clean}}^{(k)} \leq N - 1$, and the monotonicity of $h$ gives $h(m_{\textnormal{clean}}^{(k)}) \geq h(m_{\textnormal{evolve}})$. Hence
\begin{equation}
\begin{aligned}
    E[\Delta p_k]
    &= \sum_{m=0}^{N} P\big(m_{\textnormal{clean}}^{(k)} = m\big)\, h(m)\, \eta
    \;\geq\; h(m_{\textnormal{evolve}})\, \eta \, P\big(m_{\textnormal{clean}}^{(k)} \leq N-1\big) \\
    &\geq h(m_{\textnormal{evolve}})\, \eta \Big(1 - (1 - \varepsilon)^{\beta T (N - m_{\textnormal{evolve}})}\Big)
    = \delta_\varepsilon.
\end{aligned}
\end{equation}
Therefore, as long as $p_j \geq \varepsilon$ for all $j < k$, we have $p_k \leq p_0 - k \cdot \delta_\varepsilon$, so $p_k$ falls to at most $\varepsilon$ within $(p_0 - \varepsilon)/\delta_\varepsilon$ iterations and, being non-increasing, stays there thereafter. Finally, by Assumption~\ref{asmp:redundancy}, $E[\rho(\tau_K^{\textnormal{evolve}})] = \beta p_K \leq \beta \varepsilon$ for all $K \geq (p_0 - \varepsilon)/\delta_\varepsilon$.
\end{proof}

\section{Experiment Details}

\subsection{Pseudo Code}

We present \ourmodel training procedure in Algorithm~\ref{alg:trajevolve}, where the parts unique to \ourmodel are marked in \textit{italics}.
Beyond the base GRPO loop, \ourmodel introduces three key extensions:
(1) \emph{Redundancy-Aware Rewards} that penalize redundant steps in successful trajectories (Equation~\ref{equation:reward});
(2) \emph{Step-Level Advantage} via fine-grained credit assignment (Equation~\ref{equation:trajevolve}); and
(3) a periodic \emph{trajectory evolution phase} that actively refines historical trajectories to produce higher-quality training data, enabling the model to exceed its current inference limitations. The step-level credit labels used throughout the algorithm are obtained via the Hindsight Verification procedure (Algorithm~\ref{alg:forward_ablation}).
A lightweight step-level discriminator $D_\phi$ is trained via supervised fine-tuning on these credit labels and is used to identify error steps in failed trajectories during training.

\begin{algorithm}[t]
\caption{Training LLM Agents with TrajEvolve}
\label{alg:trajevolve}
\begin{algorithmic}[1]
\REQUIRE Initial policy $\pi_{\theta_{\text{old}}}$, base discriminator $D_\phi$, task distribution $p(X)$, clipping parameter $\epsilon$, group size $N$, \textit{evolution frequency $F$, credit adjustment value $\delta$, redundancy penalty weight $\alpha$}
\STATE {\color{gray}\textit{// Discriminator initialization}}
\STATE \textit{Fine-tune discriminator $D_\phi$ on credit labels obtained from Algorithm~\ref{alg:forward_ablation} with SFT}
\FOR{each training iteration $k$}
    \STATE Update the old policy model: $\theta_{\text{old}} \leftarrow \theta$
    \STATE {\color{gray}// Multi-step rollout phase}
    \STATE Sample task $x \sim p(X)$ and launch $N$ parallel environment instances
    \FOR{$t = 1$ to $T$}
        \STATE Generate actions $\{a_t^{(i)} \sim \pi_{\theta_{\text{old}}}(\cdot \mid s_t^{(i)}, x)\}_{i=1}^N$
        \STATE Execute actions, collect rewards $\{r_t^{(i)}\}_{i=1}^N$ and next states $\{s_{t+1}^{(i)}\}_{i=1}^N$
    \ENDFOR
    \STATE {\color{gray}\textit{// Credit analysis and discriminator update phase}}
    \STATE \textit{Analyze successful trajectories with Algorithm~\ref{alg:forward_ablation} to obtain credit labels}
    \STATE \textit{Fine-tune $D_\phi$ on obtained credit labels}
    \STATE \textit{Analyze failed trajectories with $D_\phi$ to identify first error step}
    \STATE {\color{gray}\textit{// Redundancy-aware rewards and Step-level advantage computation}}
    \STATE \textit{Compute credit-augmented rewards with Equation~\ref{equation:reward}}
    \STATE \textit{Compute advantages $A(a_t^{(i)})$ with Equation~\ref{equation:trajevolve}}
    \STATE {\color{gray}// Policy update phase}
    \STATE Update policy $\theta$ by maximizing objective $\mathcal{J}(\theta)$ with Equation \ref{equation:policy}
    \IF{\textit{$k \bmod F = 0$}}
        \STATE {\color{gray}\textit{// Trajectory evolution phase (Section~\ref{sec:evolution})}}
        \STATE \textit{Evolve trajectories with evolution strategies and re-execute in $\mathcal{E}$}
        \STATE \textit{Update policy $\theta$ on successful evolved trajectories}
    \ENDIF
\ENDFOR
\end{algorithmic}
\end{algorithm}

\subsection{Implementation Details}
\textbf{Hyperparameters for ALFWorld.}\quad
We keep the base RL hyperparameters consistent across all methods.
The maximum prompt length is 2048 tokens and the maximum response length is 512 tokens.
Each episode runs for up to 30 environment steps.
The actor learning rate is 1e-6.
A rule-based reward scheme assigns 10 for task completion and 0 otherwise, with an additional penalty of $-0.1$ for invalid actions.
All group-based RL methods adopt a group size of 8 with 16 groups sampled per rollout, yielding $16 \times 8 = 128$ parallel environments.
The rollout temperature is 1.0 and the validation temperature is 0.4.
The mini-batch size is set to 256.
For TrajEvolve hyperparameters, the redundancy penalty weight $\alpha$ is set to 1.0, the step-level credit adjustment value $\lambda$ is set to 0.5, and the trajectory evolution phase is triggered every $F = 5$ iterations.
The discriminator is initialized from Qwen2.5-1.5B-Instruct and fine-tuned with a learning rate of 1e-5 for 1 epoch per update.
All experiments are trained for 150 iterations on 8$\times$H20 GPUs.

\textbf{Hyperparameters for WebShop.}\quad
We keep the base RL hyperparameters consistent across all methods.
The maximum prompt length is 4096 tokens and the maximum response length is 512 tokens.
Each episode runs for up to 20 environment steps.
The actor learning rate is 1e-6.
A rule-based reward scheme assigns 10 for task completion and 0 otherwise, with an additional penalty of $-0.1$ for invalid actions.
All group-based RL methods adopt a group size of 8 with 16 groups sampled per rollout, yielding $16 \times 8 = 128$ parallel environments.
The rollout temperature is 1.0 and the validation temperature is 0.4.
The mini-batch size is set to 64.
For TrajEvolve hyperparameters, the redundancy penalty weight $\alpha$ is set to 1.0, the step-level credit adjustment value $\lambda$ is set to 0.5, and the trajectory evolution phase is triggered every 5 iterations.
The discriminator is initialized from Qwen2.5-1.5B-Instruct and fine-tuned with a learning rate of 1e-5 for 1 epoch per update.
All experiments are trained for 150 iterations on 8$\times$H20 GPUs.

\section{Additional Experiments}
\label{app:exp}

\begin{figure*}[t]
    \centering

    \begin{subfigure}[t]{0.43\linewidth}
        \centering
        \includegraphics[width=\linewidth]{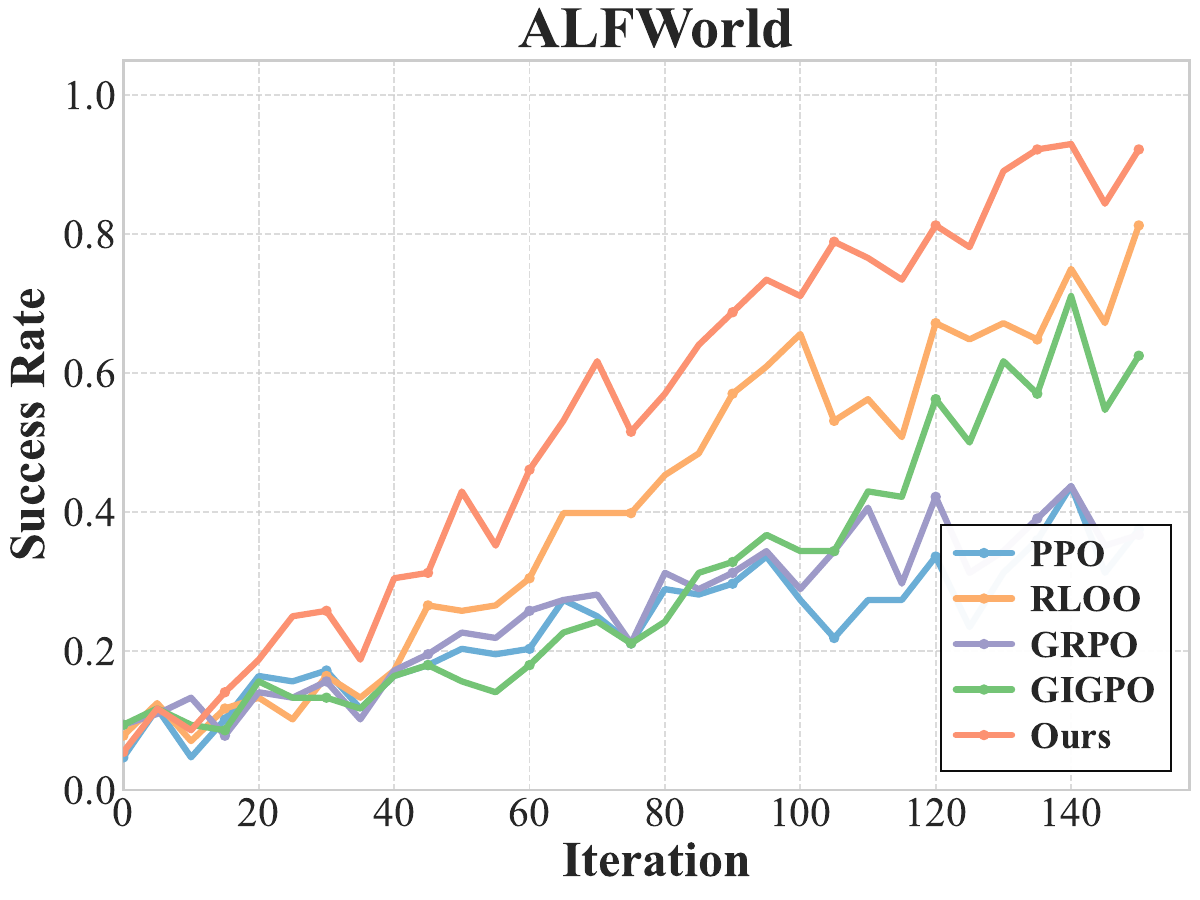}
    \end{subfigure}
    \begin{subfigure}[t]{0.43\linewidth}
        \centering
        \includegraphics[width=\linewidth]{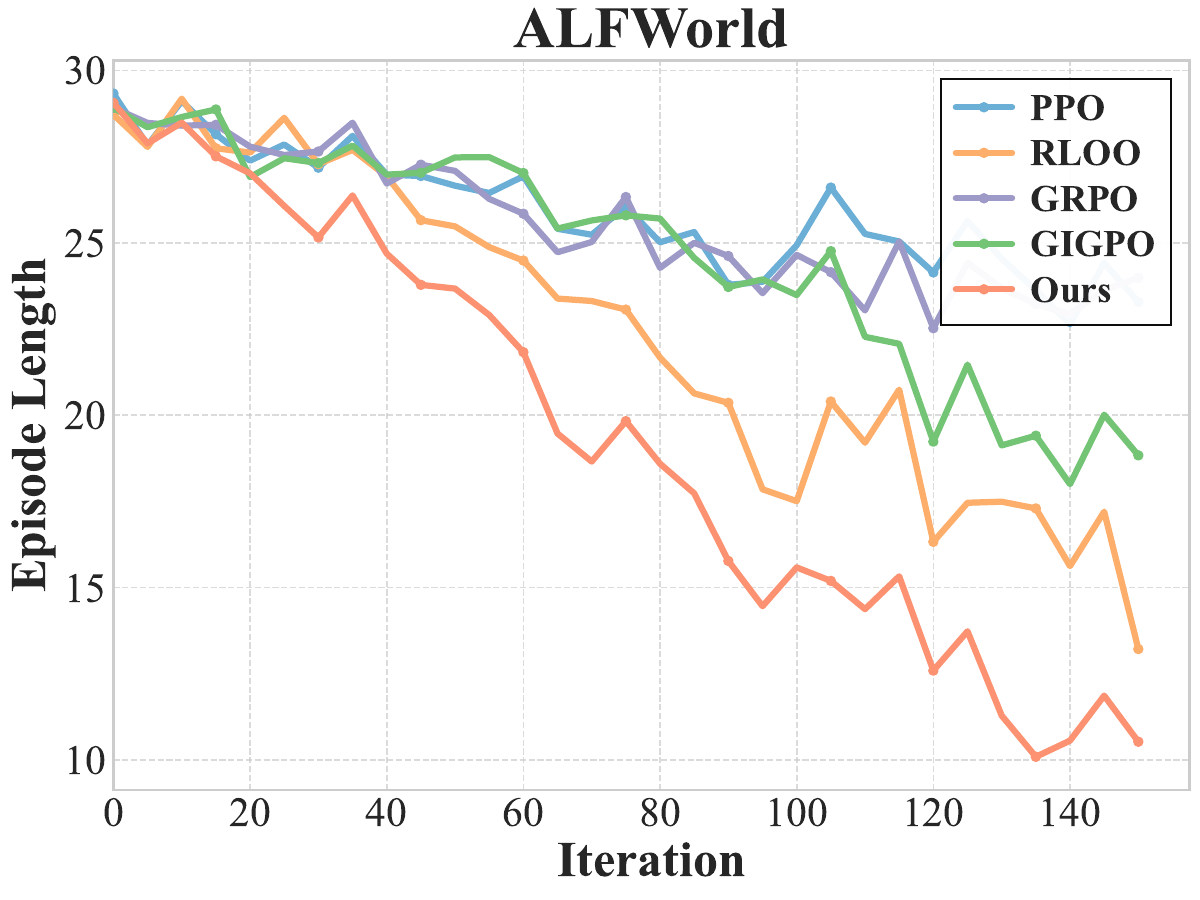}
    \end{subfigure}
    \caption{Training curves on ALFWorld with Llama-3.2-3B-Instruct. Left: validation success rate. Right: Average episode length. Our method consistently achieves a higher success rate and shorter, more efficient trajectories than all baselines.}
     \vspace{-1 em}
    \label{fig:alf-llama}
\end{figure*}

\begin{figure*}[t]
    \centering

    \begin{subfigure}[t]{0.43\linewidth}
        \centering
        \includegraphics[width=\linewidth]{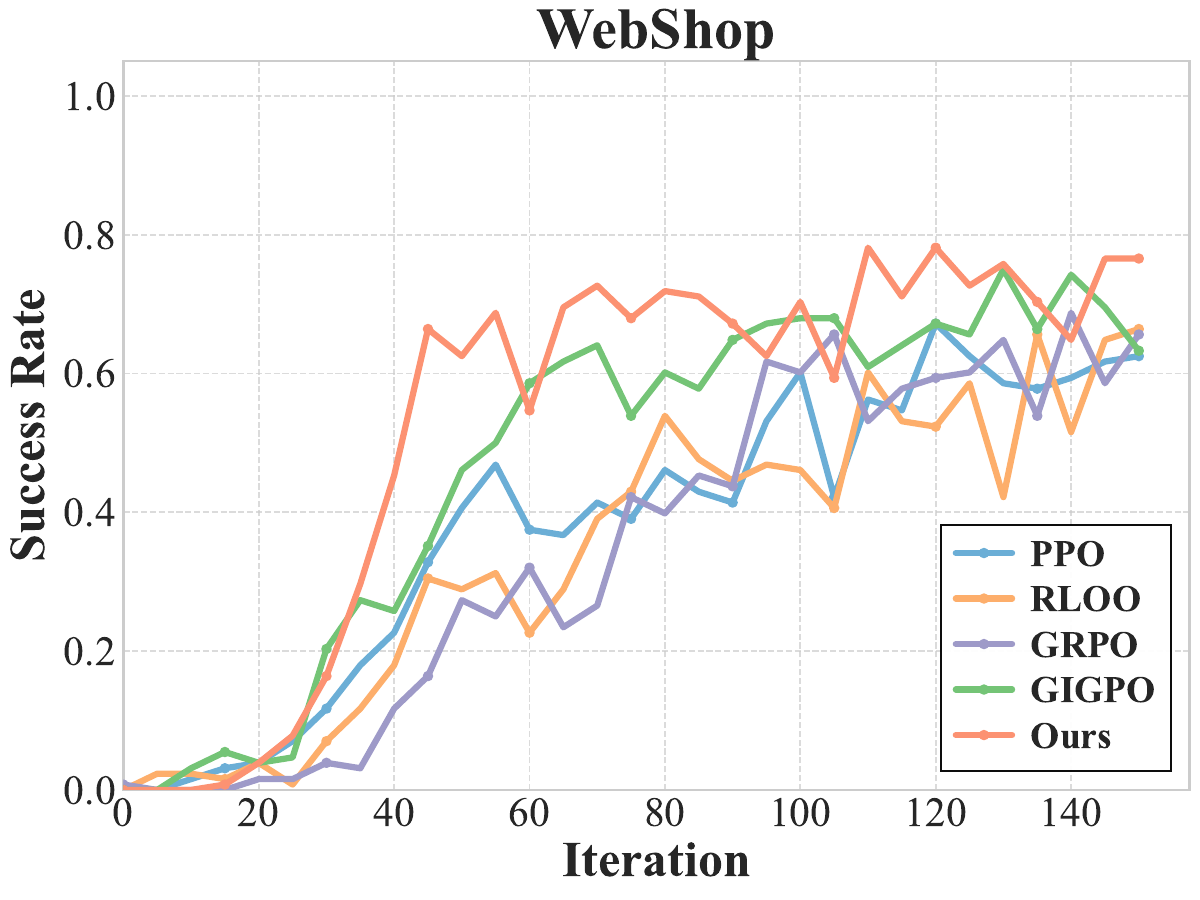}
    \end{subfigure}
    \begin{subfigure}[t]{0.43\linewidth}
        \centering
        \includegraphics[width=\linewidth]{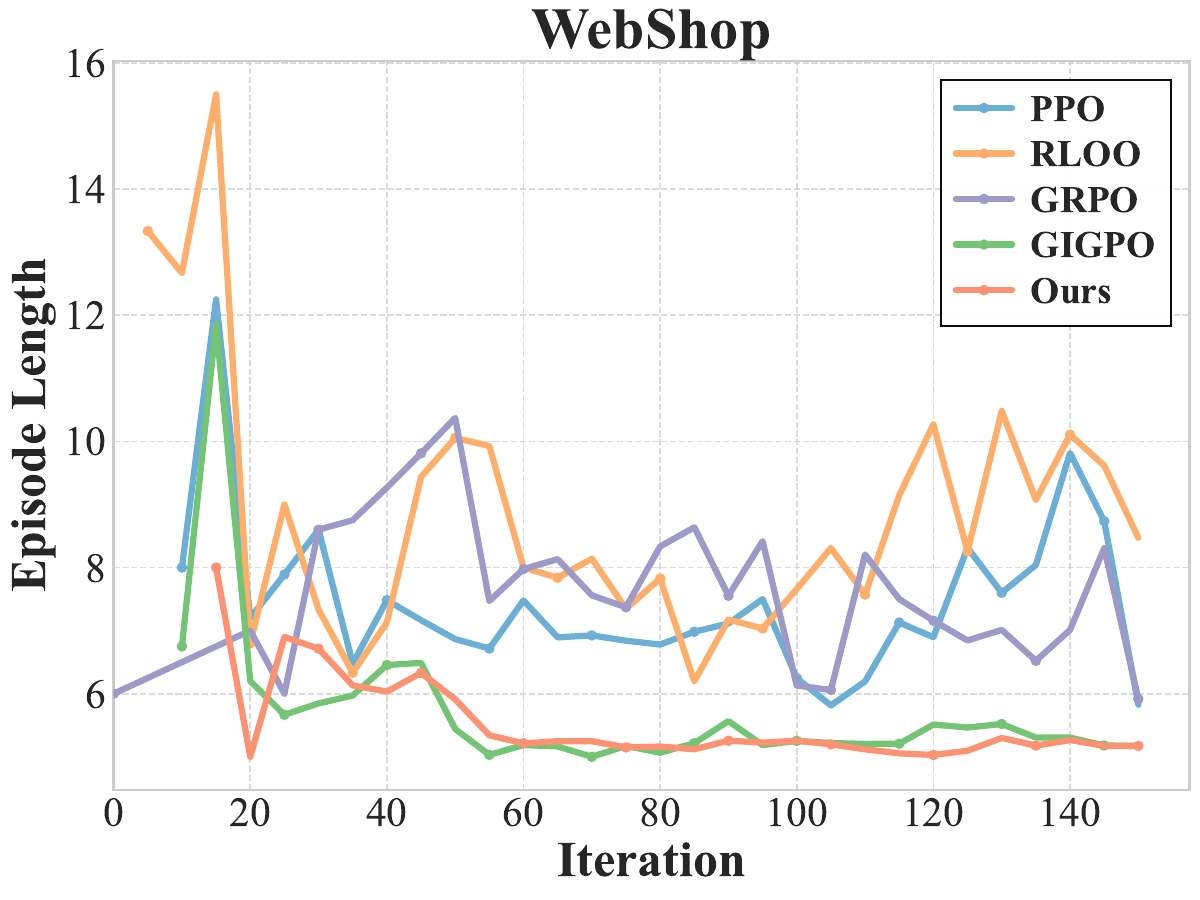}
    \end{subfigure}
    \caption{Training curves on Webshop with Llama-3.2-3B-Instruct. Left: validation success rate. Right: Average episode length. Our method consistently achieves a higher success rate and shorter, more efficient trajectories than all baselines.}
    \label{fig:web-llama}
\end{figure*}

\subsection{Base Model Ablation}
\label{sec:base_model_ablation}

To verify that TrajEvolve's effectiveness is not contingent on a specific model family, we conduct additional experiments using Llama-3.2-3B-Instruct as the base model on both ALFWorld and WebShop benchmarks. Since the main experiments in Section~\ref{sec:experiment} exclusively use the Qwen2.5 series, demonstrating consistent improvements across different model architectures strengthens the generalizability of our approach.

As shown in Table~\ref{table:llama}, our method achieves the best performance across nearly all task categories. On ALFWorld, TrajEvolve attains an overall success rate of 92.19\%, surpassing the strongest baseline RLOO at 81.25\% by a margin of 10.94\%. The improvements are particularly pronounced on task types that require longer interaction horizons, such as Cool and Heat, where our method reaches 89.47\% and 88.89\% respectively. These categories benefit most from redundancy elimination and trajectory evolution due to the larger room for refinement in longer action sequences. On WebShop, our method reaches 88.13 in score and 76.56\% in success rate, outperforming all baselines by a substantial margin. It is worth noting that PPO and GRPO perform considerably worse on ALFWorld with this base model, achieving only 37.50\% and 36.72\% overall success rates. This suggests that Llama-3.2-3B-Instruct has a weaker initial policy for embodied tasks, which exacerbates the sparse signal problem. Our method effectively alleviates this issue by injecting high-quality evolved trajectories into training.

Figure~\ref{fig:alf-llama} presents the training dynamics on ALFWorld. Our method exhibits a consistently steeper learning curve starting from around iteration 20 and maintains a clear advantage throughout training, ultimately converging to over 90\% success rate. In contrast, PPO and GRPO plateau below 50\%, while GIGPO shows unstable convergence behavior. The episode length plot further confirms that our method produces increasingly efficient trajectories, with the average episode length steadily decreasing to approximately 10 steps, the lowest among all methods. This indicates that the binary credit assignment mechanism successfully identifies and eliminates redundant actions even with a different base model.

Figure~\ref{fig:web-llama} shows analogous trends on WebShop. Our method achieves both the highest success rate and the shortest episode length. While RLOO exhibits competitive success rates, its episode length remains higher, indicating less efficient action sequences. GRPO shows notably longer episodes that peak above 13 steps in early training, consistent with the redundancy accumulation problem discussed in our theoretical analysis in Theorem~\ref{thm:main}. In contrast, our method rapidly reduces episode length to around 6 steps while simultaneously improving success rate, demonstrating that trajectory evolution and credit assignment work synergistically across model families.

Overall, these results confirm that TrajEvolve generalizes beyond the Qwen2.5 family and delivers consistent gains in both task success and trajectory efficiency when applied to Llama-3.2-3B-Instruct.

\begin{table*}[t]
\fontsize{8}{9}\selectfont 
\centering
\caption{Performance comparison on ALFWorld and WebShop with Llama-3.2-3B-Instruct, the best result in each column is highlighted in bold.} 
\label{table:llama}
\renewcommand\arraystretch{1}
\tabcolsep=0.2cm
\scalebox{1.1}{
\begin{tabular}{cccccccccc}
\toprule
\multirow{2}{*}{Method} & \multicolumn{7}{c}{ALFWorld}                           & \multicolumn{2}{c}{WebShop} \\ \cmidrule{2-8} \cmidrule{9-10}
                        & Pick  & Look  & Clean  & Heat  & Cool  & Pick2 & All   & Score        & Succ.        \\ \midrule
PPO                     & 57.14 & 25.00 & 29.17  & 27.78 & 21.05 & 41.67 & 37.50 & 86.25        & 62.50        \\
GRPO                    & 57.14 & 25.00 & 20.83  & 33.33 & 21.05 & 41.67 & 36.72 & 82.66        & 65.63        \\
RLOO                    & 94.29 & 62.50 & 95.83  & 66.67 & 52.63 & 87.50 & 81.25 & 79.81        & 66.41        \\
GIGPO                   & 65.71 & 37.50 & 66.67  & 77.78 & 52.63 & 58.33 & 62.50 & 83.04        & 63.28        \\
Ours                    & \textbf{97.14} & \textbf{62.50} & \textbf{100.00} & \textbf{88.89} & \textbf{89.47} & \textbf{91.67} & \textbf{92.19} & \textbf{88.13}        & \textbf{76.56}       \\ \bottomrule
\end{tabular}}
\vspace{-2 em}
\end{table*}

\subsection{Discriminator Accuracy Experiment}
A natural question is whether the discriminator can reliably identify necessary and redundant steps. To address this concern, we track the discriminator's classification accuracy throughout training on both benchmarks.

For successful trajectories, ground-truth credit labels are readily available through hindsight verification in the environment, so we directly compute accuracy against these labels. For failed trajectories, however, no ground-truth labels can be obtained since counterfactual verification requires a successful trajectory as input. We therefore employ GPT-5.2-2025-12-11 as an approximate judge to evaluate the discriminator's predictions on failed trajectories. To reduce computational cost, we report the accuracy over the first 50 training iterations.

As shown in Figure~\ref{fig:discriminator}, the discriminator achieves strong overall accuracy on both benchmarks, reaching 90.93\% on ALFWorld and 83.17\% on WebShop. The accuracy on successful trajectories is notably higher, at 95.06\% and 87.59\% respectively, while the accuracy on failed trajectories is lower, at 84.95\% and 77.26\%. The training curves in Figure~\ref{fig:discriminator} further show that the discriminator maintains stable and consistent accuracy throughout training, with no significant degradation as the policy evolves.

We emphasize that the discriminator's role in our framework is deliberately limited. For successful trajectories, credit assignment is performed entirely through hindsight counterfactual verification, which yields exact ground-truth labels without relying on the discriminator at all. The discriminator only serves to generalize credit signals from successful trajectories to failed ones and to guide the trajectory evolution module. This design choice explains why even a lightweight 1.5B-parameter discriminator is sufficient to achieve strong overall performance. Furthermore, we note that credit assignment for failed trajectories is an inherently difficult problem, and the GPT-5.2 judge itself may introduce evaluation noise. Therefore, the reported accuracy on failed trajectories should be interpreted as an approximate reference rather than a precise measurement. Nevertheless, the results confirm that the discriminator provides sufficiently reliable signals for trajectory evolution, and our framework's core performance is not bottlenecked by discriminator quality.

\begin{figure}[t]
    \centering
    \includegraphics[width=\linewidth]{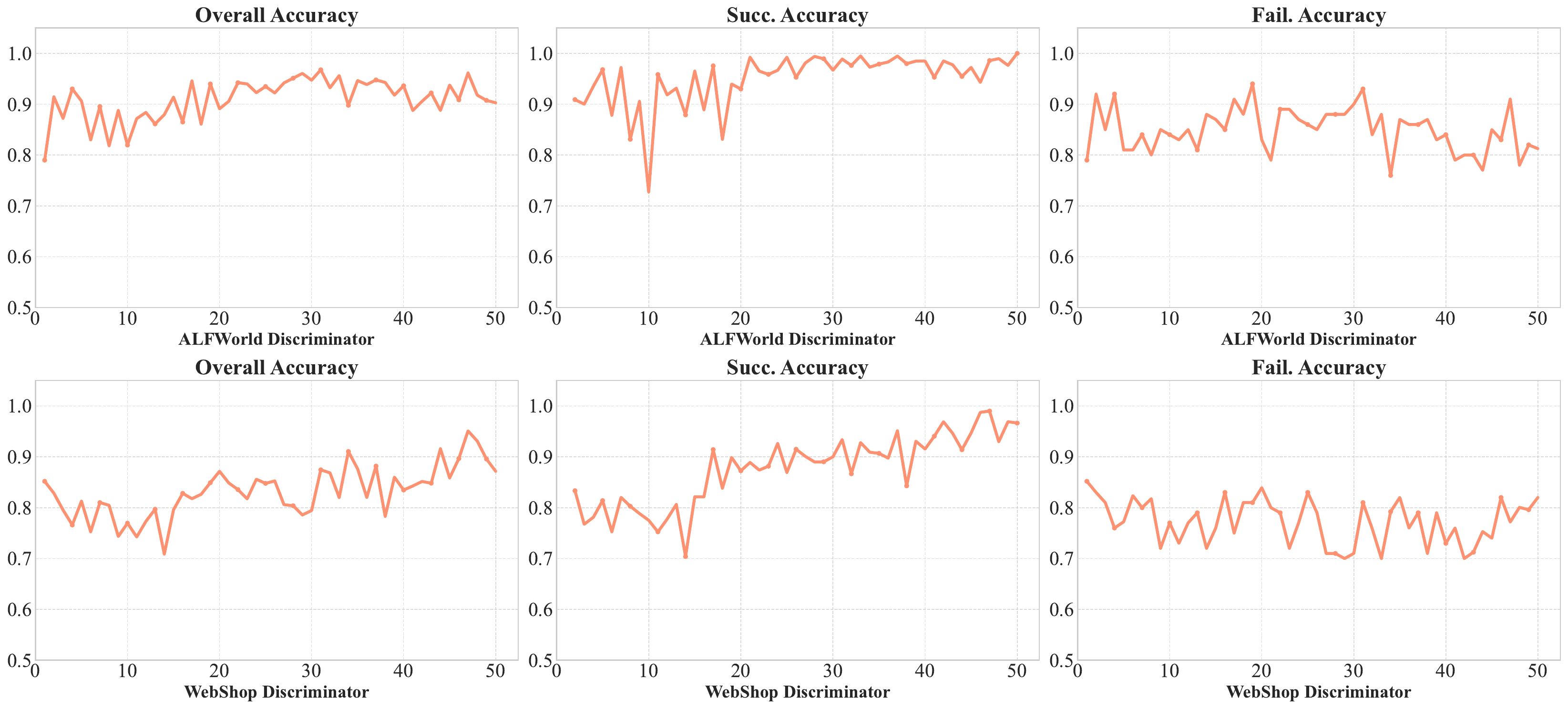}
    \caption{Training accuracy of our discriminator on ALFWorld (top) and WebShop (bottom), reporting overall, success, and failure accuracy across training iterations.}
    \label{fig:discriminator}
\end{figure}

\subsection{Computation Cost Analysis}
To assess the practical overhead introduced by TrajEvolve, we measure the per-step time cost of each training stage on both benchmarks. Understanding the computational budget is essential for evaluating whether our method can scale to larger settings without prohibitive overhead.

We profile all training stages on the same hardware configuration of 8$\times$H20 GPUs with Qwen2.5-1.5B-Instruct as the policy, and report the average per-step wall-clock time in seconds. The results are shown in Figure~\ref{fig:time}. Beyond the standard RL training pipeline that includes rollout, reference log-probability computation, and policy update, our method introduces two additional modules: hindsight verification and trajectory evolution. Note that in this setting the discriminator has the same scale as the policy, so the relative share of the discriminator-related stages reported here is an upper estimate; with a larger policy, this share shrinks accordingly.

As shown in Figure~\ref{fig:time}, the rollout stage dominates the total training time on both benchmarks, taking 117.4s on ALFWorld and 63.2s on WebShop per step. In comparison, the overhead introduced by our two additional modules is moderate. Hindsight verification costs 61.2s and 20.3s per step on ALFWorld and WebShop respectively. Although the number of verification calls scales linearly with trajectory length, each call only replays the recorded actions in the environment and checks whether the task still succeeds, without requiring any model inference. The entire process is therefore purely CPU-based and consumes no GPU resources. Moreover, the verification calls are independent of each other, so this stage is fully parallelizable across trajectories and across calls, and its wall-clock time scales down directly with the number of available CPU cores. Since our profiling uses a modest CPU budget, this overhead still has considerable room for reduction.

Trajectory evolution costs 31.6s and 27.8s per step on ALFWorld and WebShop respectively. This stage is an intrinsic, periodic cost of our method rather than a per-step tax: evolution is triggered every 5 iterations with only two trajectories selected per group, so over a 150-iteration run it adds the equivalent of roughly 30 additional steps of work in total, and in some iterations there may not be enough successful trajectories available for evolution. The evolution process is essentially equivalent to re-executing these refined trajectories in the environment, which is comparable to a small fraction of the standard rollout cost. We consider this cost well justified, as the evolved trajectories expose the policy to high-quality compositions outside its current sampling distribution, an effect that on-policy training alone cannot reproduce within a practical budget (Theorem~\ref{thm:main}) and that therefore cannot be substituted by simply allocating additional training steps.

The remaining components, discriminator training and credit labeling, introduce overhead of 18.5s and 27.2s on ALFWorld and 13.3s and 13.1s on WebShop, as the discriminator is a lightweight 1.5B-parameter model trained with LoRA for only 1 epoch per iteration. We emphasize that our method does not critically depend on a strong discriminator: credit labels for successful trajectories come entirely from counterfactual verification in the environment, which yields exact ground-truth labels without any learned estimator and thus provides abundant training signal regardless of discriminator quality. 
The discriminator can therefore be replaced with an even smaller backbone (e.g., Qwen2.5-0.5B-Instruct) to further reduce this portion of the overhead.

Overall, the total per-step time of TrajEvolve is approximately 1.5$\times$ that of standard GRPO training. A substantial part of this overhead is reducible through more CPU cores and a smaller discriminator, while the remaining periodic cost of evolution is intrinsic to the method. Given the substantial performance gains demonstrated in Section~\ref{sec:experiment}, we consider this overhead to be a favorable trade-off between computational cost and agent performance.
\begin{figure}[t]
    \centering
    \includegraphics[width=\linewidth]{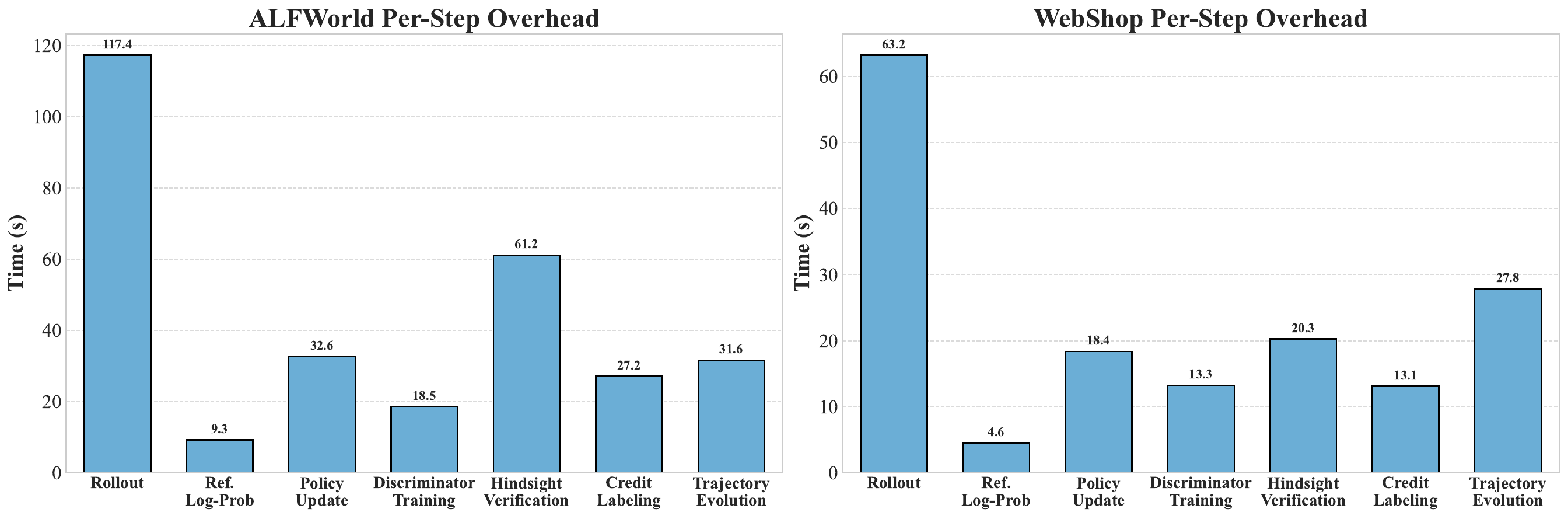}
    \caption{Per-step time overhead of each training stage on ALFWorld (left) and WebShop (right).}
    \label{fig:time}
\end{figure}

\subsection{TrajEvolve on Complex Scenarios}
\label{sec:scienceworld}
To evaluate whether TrajEvolve can generalize to more complex environments beyond ALFWorld and WebShop, we conduct additional experiments on ScienceWorld, a challenging text-based simulation environment that requires agents to perform scientific experiments through multi-step interactions. Unlike ALFWorld and WebShop where most steps directly alter the environment state, ScienceWorld involves a substantial proportion of information-gathering steps such as looking at objects, reading instructions, and using thermometers to take measurements. These steps do not change the environment or affect task completion, yet they may be essential for the agent to make informed decisions in subsequent steps.

The primary challenge in transferring TrajEvolve to this setting lies in determining the dependent steps during hindsight verification. Since information-gathering steps have no observable effect on the environment state or task outcome, removing them does not cause subsequent actions to become infeasible, nor does it directly lead to task failure. Therefore, the standard counterfactual verification procedure cannot determine whether these steps are truly necessary. To address this, we adopt a simple two-category strategy. We classify all steps into information-gathering steps and environment-changing steps based on the action verb, which provides a reliable and low-cost distinction. We retain all information-gathering steps unconditionally and apply hindsight verification only to environment-changing steps. Under this setting, our method can still effectively identify and eliminate redundant environment-changing actions, although some unnecessary information-gathering steps may be preserved. This simple adaptation allows TrajEvolve to transfer to complex scenarios with minimal modification.

\begin{wraptable}{r}{0.42\textwidth}
    \vspace{-12pt}
    \centering
    \caption{Final success rate (\%) on ScienceWorld with Qwen2.5-7B-Instruct after 150 iterations. IID: seen task variants; OOD: unseen variants within seen categories.}
    \label{tab:scienceworld}
    \small
    \begin{tabular}{lcc}
        \toprule
        Method & IID & OOD \\
        \midrule
        GRPO   & 92.19 & 86.72 \\
        GiGPO  & 92.97 & 87.50 \\
        Ours   & \textbf{95.31} & \textbf{89.06} \\
        \bottomrule
    \end{tabular}
    \vspace{-10pt}
\end{wraptable}
We train all methods using Qwen2.5-7B-Instruct for 150 iterations and evaluate on two settings: IID, which tests on seen task variants and categories, and OOD, which tests on unseen task variants within seen categories. As shown in Table~\ref{tab:scienceworld}, our method achieves the highest performance on both settings, reaching 95.31\% on IID and 89.06\% on OOD, compared with 92.19\% and 86.72\% for GRPO and 92.97\% and 87.50\% for GiGPO. The improvements on the OOD setting are particularly encouraging, as they suggest that the more efficient trajectories learned through our evolution mechanism also benefit generalization to unseen task variants. These results confirm that TrajEvolve can be effectively transferred to more complex environments through lightweight adaptation of the verification procedure, without requiring fundamental changes to the framework.

\section{Method Analysis}
\label{app:analysis}
\subsection{Method Transferability Analysis}
We discuss the applicability of TrajEvolve to different task settings and clarify the conditions under which our method is most effective.

The core requirement of TrajEvolve is that the environment supports trajectory replay, meaning that a sequence of actions can be deterministically re-executed to reproduce the same state transitions. This property is widely satisfied across many practical domains. Task execution environments such as ALFWorld and WebShop are inherently replayable since the environment state is fully determined by the action history. Code generation and debugging scenarios also naturally support replay, as programs can be re-executed with the same inputs to verify correctness. Web search and retrieval tasks similarly allow re-issuing queries and re-processing results in a deterministic manner. Despite this broad applicability, existing RL methods for LLM agents have largely overlooked the replay property of these environments. TrajEvolve explicitly leverages this property to construct higher-quality trajectories through hindsight verification and trajectory evolution, providing additional training signals that go beyond what on-policy sampling alone can offer.

Our method is particularly advantageous in environments where environment-changing actions constitute a large proportion of the trajectory. For such actions, counterfactual verification can directly determine whether each step is genuinely necessary for task completion by observing changes in the environment state and task outcome. This yields ground-truth credit labels without relying on any learned estimator, which is a unique strength of our approach. The strong experimental results on ALFWorld, where most steps involve physical state changes, exemplify this advantage.

However, this does not mean that TrajEvolve lacks transferability to more complex scenarios involving substantial information-gathering steps. As demonstrated in Section~\ref{sec:scienceworld}, a simple two-category adaptation that retains information-gathering steps and applies verification only to environment-changing steps allows our method to achieve consistent improvements on ScienceWorld. Although this adaptation is conservative in that it preserves some potentially redundant information-gathering steps, it maintains the core benefits of trajectory evolution while requiring minimal modification to the framework. Exploring more fine-grained verification strategies for information-gathering steps remains a promising direction for future work.

\subsection{Off-Policy Analysis}
\begin{figure}[t]
    \centering
    \includegraphics[width=0.85\linewidth]{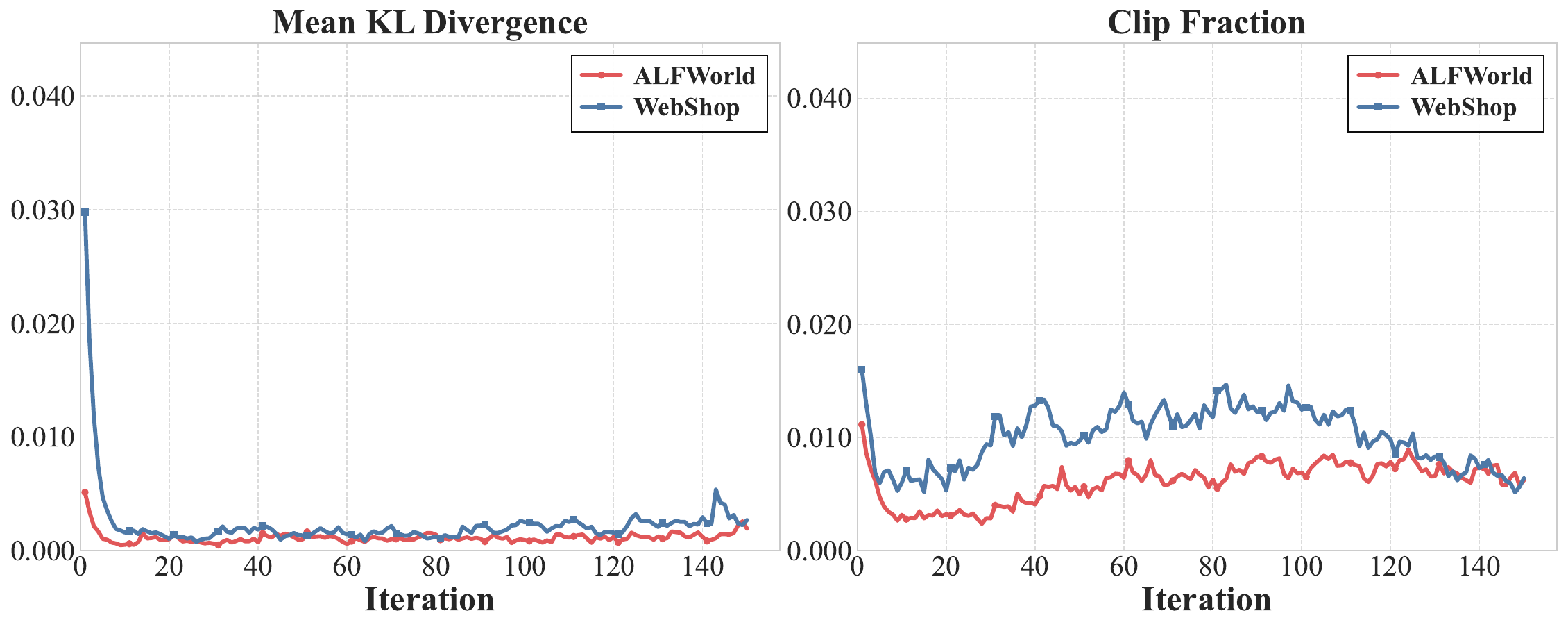}
    \caption{Mean KL divergence (left) and clip fraction (right) during training on ALFWorld and WebShop, indicating stable policy updates.}
    \label{fig:off-training}
\end{figure}

\begin{figure}[t]
    \centering
    \includegraphics[width=0.85\linewidth]{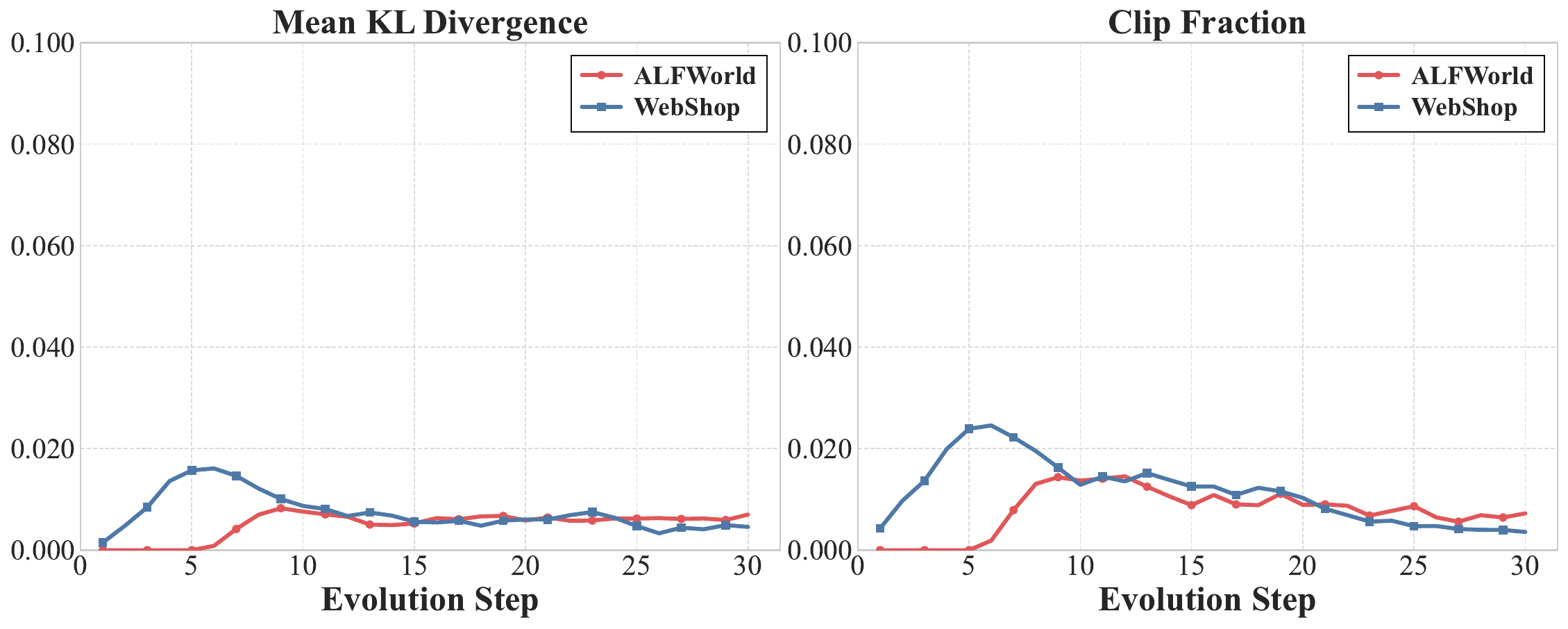}
    \caption{Mean KL divergence (left) and clip fraction (right) during the trajectory evolution stage on ALFWorld and WebShop.}
    \label{fig:off-evolution}
\end{figure}
Since evolved trajectories are generated by refining historical rollouts rather than sampled from the current policy, they are inherently off-policy. A natural concern is whether incorporating such trajectories into PPO-style optimization introduces excessive distributional shift that could destabilize training. To investigate this, we monitor two standard trust-region diagnostics throughout training on both benchmarks.

The first metric is the approximate KL divergence between the updated policy and the reference policy, which measures the magnitude of per-update policy shift. The second metric is the clip fraction, which measures the proportion of tokens whose importance sampling ratio falls outside the clipping range $[1-\epsilon, 1+\epsilon]$. A high clip fraction indicates that many tokens experience large probability changes, suggesting significant distributional mismatch. We note that before each evolution update, we recompute log-probabilities under the current policy as the baseline for the importance sampling ratio. This ensures that the clipping mechanism correctly constrains the policy shift regardless of how the trajectories were originally generated.

Figure~\ref{fig:off-training} presents the two metrics during the standard training phase across 150 iterations. On both benchmarks, the KL divergence drops rapidly within the first 20 iterations and stabilizes at very low values below 0.005 for the remainder of training. The clip fraction on ALFWorld remains consistently low throughout, while on WebShop it stabilizes around 0.01 after an initial transient. These patterns indicate that the overall training process maintains stable policy updates even with the periodic injection of evolved trajectories.

Figure~\ref{fig:off-evolution} focuses specifically on the trajectory evolution phases, showing the two metrics across the 30 evolution steps that occur once every 5 training iterations. Both KL divergence and clip fraction exhibit a brief peak during the early evolution steps when the policy is changing most rapidly, and then steadily decrease over subsequent phases. By the later stages, both metrics converge below 0.005 and 0.01 respectively, reaching levels comparable to standard on-policy PPO updates. The average KL divergence across all evolution phases is 0.0054 on ALFWorld and 0.0075 on WebShop, while the average clip fraction is 0.0082 and 0.0115 respectively. These results confirm that the distributional shift introduced by trajectory evolution is bounded and small, and that the shift diminishes as training progresses since evolved trajectories become increasingly aligned with the current policy.

We note that leveraging off-policy successful trajectories has emerged as a promising direction in recent RL research. For instance, ExGRPO~\citep{rl_exgrpo} replays successful rollouts for additional training updates. Our empirical analysis provides evidence that such off-policy utilization can be safely combined with PPO-style optimization through proper importance sampling and clipping. Exploring more efficient use of off-policy data in long-horizon agent training remains a valuable direction for future work.

\section{Case Study}

\begin{table*}[t]
\centering
\small
\caption{\textbf{Case study of counterfactual credit analysis on a successful ALFWorld trajectory.}
The agent is asked to \emph{put a clean knife in sidetable}. Each step $t$ is
removed from the trajectory (together with steps already identified as redundant)
and the remaining actions are replayed in the environment: $t$ is labeled
\necTag{} if the task then fails, and \redTag{} otherwise. $\mathrm{Dep}(t)$ is the
set of later steps whose actions become inexecutable (\textit{``Nothing happens.''})
in this replay. When $t$ is redundant, all steps in $\mathrm{Dep}(t)$ are marked
redundant as well without further replay ($^\dagger$). Steps 0--3 are unsuccessful
drawer searches and are all redundant; steps 4--9 form the necessary chain
(\textit{locate $\rightarrow$ grasp $\rightarrow$ clean $\rightarrow$ deliver}).
These labels supervise the training of our step discriminator.}
\label{tab:case_study_alfworld}
\resizebox{\linewidth}{!}{
\renewcommand\arraystretch{1}
\tabcolsep=0.2cm
\scalebox{0.8}{
\begin{tabular}{@{}c l p{6.6cm} c c@{}}
\toprule
\textbf{Step} & \textbf{Action} & \textbf{Observation} & \textbf{Credit} & $\mathrm{Dep}(t)$ \\
\midrule
0 & go to drawer 1                & The drawer 1 is closed.                                                                                       & \redTag & \{1\} \\
1 & open drawer 1                 & The drawer 1 is open. In it, you see nothing.                                                                 & \redTag & $\emptyset$ \\
2 & go to drawer 2                & The drawer 2 is closed.                                                                                       & \redTag & \{3\} \\
3 & open drawer 2                 & The drawer 2 is open. In it, you see a spoon 1.                                                               & \redTag & $\emptyset$ \\
4 & go to countertop 1            & On the countertop 1, you see \ldots a knife 1 \ldots (and 18 other objects).                                  & \necTag & \{5,7,9\} \\
5 & take knife 1 from countertop 1 & You pick up the knife 1 from the countertop 1.                                                               & \necTag & \{7,9\} \\
6 & go to sinkbasin 1             & On the sinkbasin 1, you see a potato 2 and a tomato 3.                                                        & \necTag & \{7\} \\
7 & clean knife 1 with sinkbasin 1 & You clean the knife 1 using the sinkbasin 1.                                                                 & \necTag & $\emptyset$ \\
8 & go to sidetable 1             & On the sidetable 1, you see a mug 1, a papertowelroll 1, and a saltshaker 1.                                  & \necTag & \{9\} \\
9 & move knife 1 to sidetable 1   & You move the knife 1 to the sidetable 1.\hfill\textbf{[task done]}                                            & \necTag & $\emptyset$ \\
\bottomrule
\end{tabular}}}
\end{table*}

Table~\ref{tab:case_study_alfworld} illustrates how step-level credit labels are
obtained for a successful ALFWorld trajectory via counterfactual verification.
For each step $t$, we remove it from the trajectory, replay the remaining actions
in the environment, and label it \necTag{} if the task can no longer be completed
and \redTag{} otherwise. The same replay also yields $\mathrm{Dep}(t)$, the set of
subsequent steps whose actions become inexecutable once step $t$ is removed.
Three observations are worth highlighting:
(i) The agent first explores two unrelated drawers (steps~0--3). Removing step~0
or step~2 leaves the task completable, so both are redundant; their dependents,
the corresponding \texttt{open} actions (steps~1 and~3), are thereby marked
redundant without additional replays. Hence four perfectly executable actions
receive no credit, a distinction that the outcome reward alone cannot make.
(ii) Removing any of the remaining six steps causes the task to fail, so all of
them are necessary and together form a compact success chain
(\textit{locate} $\to$ \textit{grasp} $\to$ \textit{clean} $\to$ \textit{deliver}).
(iii) $\mathrm{Dep}(t)$ exposes a dependency structure aligned with the sub-goals of
the task: removing the visit to the countertop (step~4) invalidates taking the
knife (step~5), cleaning it (step~7), and the final placement (step~9), whereas the
sink-basin visit (step~6) only enables the cleaning step. Notably, steps~7 and~9
have empty dependency sets yet are still necessary: removing them leaves every
later action executable but the goal unsatisfied. Environment-level success
verification and action-level validity checks are therefore complementary.
The resulting labels are used directly for credit assignment on successful
trajectories, and serve as supervision for the step discriminator, which is
applied to failed trajectories, where counterfactual verification is not
applicable, to locate the first erroneous step.

\section{Extended Related Work}
RL fine-tuning of large language models~\citep{rl_rlvmr, rl_arpo, rl_spo, rl_rlver, rl_istar, criticsearch, acpo, carl, hz_hindsight, qr_gca, yx_finhopper, zs_r2-searcher, xdr_harnessing} was popularized by RLHF~\citep{prime, pure, rl_spo, stepsearch, hindsight} with PPO~\citep{ouyang2022training,bai2022training}, where on-policy actor-critic updates are used to optimize against a learned reward model. To remove the value network and reduce the memory and compute footprint, GRPO~\citep{rl_grpo} estimates advantages from groups of responses sharing the same prompt, and has since become the de facto recipe behind reasoning-oriented LLMs such as DeepSeek-R1~\citep{guo2025deepseek}, Qwen2.5-Math~\citep{yang2024qwen2}, and Kimi-K1.5~\citep{team2025kimi}. A growing line of work refines GRPO by addressing its bias and stability issues: Dr.\,GRPO removes length and difficulty biases in advantage normalization~\citep{rl_drgrpo}, DAPO introduces decoupled clipping and dynamic sampling for long-CoT training~\citep{rl_dapo}, GSPO performs sequence-level importance weighting to reduce variance~\citep{rl_gspo}, and VinePPO and VAPO revisit token-level credit assignment with Monte-Carlo value estimates~\citep{rl_vineppo,yue2025vapo}. RLOO and ReMax similarly exploit multi-sample baselines to obtain low-variance, value-network-free updates~\citep{rl_rloo,li2023remax}. Because each rollout from an LLM is expensive, recent work increasingly emphasizes \emph{reusing off-policy experience}, which can be viewed as the LLM analogue of experience replay. ReST, ReST-EM, and RAFT iteratively collect, filter, and replay self-generated trajectories for reward-weighted updates~\citep{gulcehre2023reinforced,singh2023beyond,dong2023raft, lare, Implicitprm}, while Self-Rewarding LMs and SPIN bootstrap a replay-style buffer of model completions for iterative DPO-style training~\citep{yuan2024self,chen2024self}. Several works push GRPO itself toward the off-policy regime: Tapered Off-Policy REINFORCE and LUFFY incorporate stale or expert trajectories with importance correction and clipping~\citep{roux2025tapered,yan2025learning}, and asynchronous trajectory-balance training decouples data collection from policy updates to better utilize a replay buffer~\citep{bartoldson2025trajectory}. In multi-turn agentic RL~\citep{hu2025training, rl_tree-grpo, rl_rest, rl_igpo, rl_portool, rl_gepo, spa-rl}, replay buffers of past trajectories have been shown to stabilize long-horizon credit assignment in environments such as ALFWorld and WebShop~\citep{zhou2024archer,song2024trial,xi2025agentgym}. These developments collectively position experience replay as a key lever for sample-efficient and scalable RL post-training of LLMs.

To improve sample efficiency in online RL, replay-based methods~\citep{rl_exgrpo, rl_arpo, rl_off_policy_GRPO} store and reuse historical rollouts from replay buffers, reducing redundant sampling. These approaches typically prioritize valuable experiences based on correctness or entropy.

\section{Prompts}

We adopt the same prompt design as GiGPO~\citep{rl_gigpo} for both ALFWorld and WebShop agents.
As shown in Figure~\ref{fig:prompt_alfworld} and Figure~\ref{fig:prompt_webshop}, each prompt provides the agent with the task description, recent interaction history, current observation, and the set of admissible actions.
The agent is instructed to first produce step-by-step reasoning within \texttt{\color{teal}<think>} \texttt{\color{teal}</think>} tags, followed by the selected action within \texttt{\color{violet}<action>} \texttt{\color{violet}</action>} tags.
We set the history length to 5 for both environments.

In addition to the standard prompts, we introduce \emph{think-only} prompt templates (Figure~\ref{fig:prompt_alfworld_think_only} and Figure~\ref{fig:prompt_webshop_think_only}) specifically for the trajectory evolution phase.
During evolution, certain steps are assigned predetermined actions that are either retained from the original trajectory or corrected via format mapping, and the model only needs to generate the corresponding reasoning.
The think-only prompt replaces the original action instruction block with the predetermined action and asks the model to produce only the reasoning content, without outputting any action tags.
This ensures that evolved trajectories maintain coherent thought processes for every step, even when the actions are externally specified rather than generated by the policy.

\clearpage

\vspace{10 em}
\begin{figure}[t]
\begin{tcolorbox}[
  enhanced,
  title={\textbf{Prompt Template for ALFWorld}},
  colback=white,
  colframe=black!80,
  coltitle=white,
  colbacktitle=black!80,
  fonttitle=\sffamily\small,
  boxrule=0.8pt,
  arc=3mm,
  left=6pt, right=6pt, top=4pt, bottom=4pt,
  breakable
]
\small
You are an expert agent operating in the ALFRED Embodied Environment. Your task is to: \texttt{\color{brown}\{task\_description\}}. Prior to this step, you have already taken \texttt{\color{brown}\{step\_count\}} step(s). Below are the most recent \texttt{\color{brown}\{history\_length\}} observations and the corresponding actions you took: \texttt{\color{brown}\{action\_history\}}. You are now at step \texttt{\color{brown}\{current\_step\}} and your current observation is: \texttt{\color{brown}\{current\_observation\}}. Your admissible actions of the current situation are: [\texttt{\color{brown}\{admissible\_actions\}}].

Now it's your turn to take an action. You should first reason step-by-step about the current situation. This reasoning process MUST be enclosed within \texttt{\color{teal}<think>} \texttt{\color{teal}</think>} tags. Once you've finished your reasoning, you should choose an admissible action for current step and present it within \texttt{\color{violet}<action>} \texttt{\color{violet}</action>} tags.
\end{tcolorbox}
\caption{The prompt template of ALFWorld agents.}
\label{fig:prompt_alfworld}
\end{figure}

\vspace{10 em}

\begin{figure}[t]
\begin{tcolorbox}[
  enhanced,
  title={\textbf{Prompt Template for WebShop}},
  colback=white,
  colframe=black!80,
  coltitle=white,
  colbacktitle=black!80,
  fonttitle=\sffamily\small,
  boxrule=0.8pt,
  arc=3mm,
  left=6pt, right=6pt, top=4pt, bottom=4pt,
  breakable
]
\small
You are an expert autonomous agent operating in the WebShop e-commerce environment. Your task is to: \texttt{\color{brown}\{task\_description\}}. Prior to this step, you have already taken \texttt{\color{brown}\{step\_count\}} step(s). Below are the most recent \texttt{\color{brown}\{history\_length\}} observations and the corresponding actions you took: \texttt{\color{brown}\{action\_history\}}. You are now at step \texttt{\color{brown}\{current\_step\}} and your current observation is: \texttt{\color{brown}\{current\_observation\}}. Your admissible actions of the current situation are: [\texttt{\color{brown}\{available\_actions\}}].

Now it's your turn to take one action for the current step. You should first reason step-by-step about the current situation, then think carefully which admissible action best advances the shopping goal. This reasoning process MUST be enclosed within \texttt{\color{teal}<think>} \texttt{\color{teal}</think>} tags. Once you've finished your reasoning, you should choose an admissible action for current step and present it within \texttt{\color{violet}<action>} \texttt{\color{violet}</action>} tags.
\end{tcolorbox}
\caption{The prompt template of WebShop agents.}
\label{fig:prompt_webshop}
\end{figure}

\definecolor{varcolor}{RGB}{30, 95, 200}

\begin{figure}[t]
\begin{tcolorbox}[
  enhanced,
  title={\textbf{Think-Only Prompt Template for ALFWorld}},
  colback=white,
  colframe=black!80,
  coltitle=white,
  colbacktitle=black!80,
  fonttitle=\sffamily\small,
  boxrule=0.8pt,
  arc=3mm,
  left=6pt, right=6pt, top=4pt, bottom=4pt,
  breakable
]
\small
You are an expert agent operating in the ALFRED Embodied Environment. Your task is to: \texttt{\color{brown}\{task\_description\}}. Prior to this step, you have already taken \texttt{\color{brown}\{step\_count\}} step(s). Below are the most recent \texttt{\color{brown}\{history\_length\}} observations and the corresponding actions you took: \texttt{\color{brown}\{action\_history\}}. You are now at step \texttt{\color{brown}\{current\_step\}} and your current observation is: \texttt{\color{brown}\{current\_observation\}}. Your admissible actions of the current situation are: [\texttt{\color{brown}\{admissible\_actions\}}].

The action to take is: \texttt{\color{varcolor}\{evolution\_action\}}.
Generate only your reasoning for why this action is appropriate. Do not output any action tags, just the reasoning content.
\end{tcolorbox}
\caption{The think-only prompt template used during the trajectory evolution phase in ALFWorld. The predetermined action is provided and the model is asked to generate only the reasoning.}
\label{fig:prompt_alfworld_think_only}
\end{figure}

\begin{figure}[t]
\begin{tcolorbox}[
  enhanced,
  title={\textbf{Think-Only Prompt Template for WebShop}},
  colback=white,
  colframe=black!80,
  coltitle=white,
  colbacktitle=black!80,
  fonttitle=\sffamily\small,
  boxrule=0.8pt,
  arc=3mm,
  left=6pt, right=6pt, top=4pt, bottom=4pt,
  breakable
]
\small
You are an expert autonomous agent operating in the WebShop e-commerce environment. Your task is to: \texttt{\color{brown}\{task\_description\}}. Prior to this step, you have already taken \texttt{\color{brown}\{step\_count\}} step(s). Below are the most recent \texttt{\color{brown}\{history\_length\}} observations and the corresponding actions you took: \texttt{\color{brown}\{action\_history\}}. You are now at step \texttt{\color{brown}\{current\_step\}} and your current observation is: \texttt{\color{brown}\{current\_observation\}}. Your admissible actions of the current situation are: [\texttt{\color{brown}\{available\_actions\}}].

The action to take is: \texttt{\color{varcolor}\{evolution\_action\}}. Generate only your reasoning for why this action is appropriate. Do not output any action tags, just the reasoning content.
\end{tcolorbox}
\caption{The think-only prompt template used during the trajectory evolution phase in WebShop. The predetermined action is provided and the model is asked to generate only the reasoning.}
\label{fig:prompt_webshop_think_only}
\end{figure}

\clearpage

%%%%%%%%%%%%%%%%%%%%%%%%%%%%%%%%%%%%%%%%%%%%%%%%%%%%%%%%%%%%

\newpage
\section*{NeurIPS Paper Checklist}

\begin{enumerate}

\item {\bf Claims}
    \item[] Question: Do the main claims made in the abstract and introduction accurately reflect the paper's contributions and scope?
    \item[] Answer: \answerYes{} % Replace by \answerYes{}, \answerNo{}, or \answerNA{}.
    \item[] Justification: The main claims are stated in the abstract and introduction
    \item[] Guidelines:
    \begin{itemize}
        \item The answer \answerNA{} means that the abstract and introduction do not include the claims made in the paper.
        \item The abstract and/or introduction should clearly state the claims made, including the contributions made in the paper and important assumptions and limitations. A \answerNo{} or \answerNA{} answer to this question will not be perceived well by the reviewers. 
        \item The claims made should match theoretical and experimental results, and reflect how much the results can be expected to generalize to other settings. 
        \item It is fine to include aspirational goals as motivation as long as it is clear that these goals are not attained by the paper. 
    \end{itemize}

\item {\bf Limitations}
    \item[] Question: Does the paper discuss the limitations of the work performed by the authors?
    \item[] Answer: \answerYes{} % Replace by \answerYes{}, \answerNo{}, or \answerNA{}.
    \item[] Justification: Please refer to Section 5.
    \item[] Guidelines:
    \begin{itemize}
        \item The answer \answerNA{} means that the paper has no limitation while the answer \answerNo{} means that the paper has limitations, but those are not discussed in the paper. 
        \item The authors are encouraged to create a separate ``Limitations'' section in their paper.
        \item The paper should point out any strong assumptions and how robust the results are to violations of these assumptions (e.g., independence assumptions, noiseless settings, model well-specification, asymptotic approximations only holding locally). The authors should reflect on how these assumptions might be violated in practice and what the implications would be.
        \item The authors should reflect on the scope of the claims made, e.g., if the approach was only tested on a few datasets or with a few runs. In general, empirical results often depend on implicit assumptions, which should be articulated.
        \item The authors should reflect on the factors that influence the performance of the approach. For example, a facial recognition algorithm may perform poorly when image resolution is low or images are taken in low lighting. Or a speech-to-text system might not be used reliably to provide closed captions for online lectures because it fails to handle technical jargon.
        \item The authors should discuss the computational efficiency of the proposed algorithms and how they scale with dataset size.
        \item If applicable, the authors should discuss possible limitations of their approach to address problems of privacy and fairness.
        \item While the authors might fear that complete honesty about limitations might be used by reviewers as grounds for rejection, a worse outcome might be that reviewers discover limitations that aren't acknowledged in the paper. The authors should use their best judgment and recognize that individual actions in favor of transparency play an important role in developing norms that preserve the integrity of the community. Reviewers will be specifically instructed to not penalize honesty concerning limitations.
    \end{itemize}

\item {\bf Theory assumptions and proofs}
    \item[] Question: For each theoretical result, does the paper provide the full set of assumptions and a complete (and correct) proof?
    \item[] Answer: \answerNA{} % Replace by \answerYes{}, \answerNo{}, or \answerNA{}.
    \item[] Justification: We have given the full assumptions and complete proofs.
    \item[] Guidelines:
    \begin{itemize}
        \item The answer \answerNA{} means that the paper does not include theoretical results. 
        \item All the theorems, formulas, and proofs in the paper should be numbered and cross-referenced.
        \item All assumptions should be clearly stated or referenced in the statement of any theorems.
        \item The proofs can either appear in the main paper or the supplemental material, but if they appear in the supplemental material, the authors are encouraged to provide a short proof sketch to provide intuition. 
        \item Inversely, any informal proof provided in the core of the paper should be complemented by formal proofs provided in appendix or supplemental material.
        \item Theorems and Lemmas that the proof relies upon should be properly referenced. 
    \end{itemize}

    \item {\bf Experimental result reproducibility}
    \item[] Question: Does the paper fully disclose all the information needed to reproduce the main experimental results of the paper to the extent that it affects the main claims and/or conclusions of the paper (regardless of whether the code and data are provided or not)?
    \item[] Answer: \answerYes{} % Replace by \answerYes{}, \answerNo{}, or \answerNA{}.
    \item[] Justification: We have included all the information needed to reproduce the results.
    \item[] Guidelines:
    \begin{itemize}
        \item The answer \answerNA{} means that the paper does not include experiments.
        \item If the paper includes experiments, a \answerNo{} answer to this question will not be perceived well by the reviewers: Making the paper reproducible is important, regardless of whether the code and data are provided or not.
        \item If the contribution is a dataset and\slash or model, the authors should describe the steps taken to make their results reproducible or verifiable. 
        \item Depending on the contribution, reproducibility can be accomplished in various ways. For example, if the contribution is a novel architecture, describing the architecture fully might suffice, or if the contribution is a specific model and empirical evaluation, it may be necessary to either make it possible for others to replicate the model with the same dataset, or provide access to the model. In general. releasing code and data is often one good way to accomplish this, but reproducibility can also be provided via detailed instructions for how to replicate the results, access to a hosted model (e.g., in the case of a large language model), releasing of a model checkpoint, or other means that are appropriate to the research performed.
        \item While NeurIPS does not require releasing code, the conference does require all submissions to provide some reasonable avenue for reproducibility, which may depend on the nature of the contribution. For example
        \begin{enumerate}
            \item If the contribution is primarily a new algorithm, the paper should make it clear how to reproduce that algorithm.
            \item If the contribution is primarily a new model architecture, the paper should describe the architecture clearly and fully.
            \item If the contribution is a new model (e.g., a large language model), then there should either be a way to access this model for reproducing the results or a way to reproduce the model (e.g., with an open-source dataset or instructions for how to construct the dataset).
            \item We recognize that reproducibility may be tricky in some cases, in which case authors are welcome to describe the particular way they provide for reproducibility. In the case of closed-source models, it may be that access to the model is limited in some way (e.g., to registered users), but it should be possible for other researchers to have some path to reproducing or verifying the results.
        \end{enumerate}
    \end{itemize}

\item {\bf Open access to data and code}
    \item[] Question: Does the paper provide open access to the data and code, with sufficient instructions to faithfully reproduce the main experimental results, as described in supplemental material?
    \item[] Answer: \answerYes{} % Replace by \answerYes{}, \answerNo{}, or \answerNA{}.
    \item[] Justification: We have released the code. In addition, the datasets used in our work are all public datasets.
    \item[] Guidelines:
    \begin{itemize}
        \item The answer \answerNA{} means that paper does not include experiments requiring code.
        \item Please see the NeurIPS code and data submission guidelines (\url{https://neurips.cc/public/guides/CodeSubmissionPolicy}) for more details.
        \item While we encourage the release of code and data, we understand that this might not be possible, so \answerNo{} is an acceptable answer. Papers cannot be rejected simply for not including code, unless this is central to the contribution (e.g., for a new open-source benchmark).
        \item The instructions should contain the exact command and environment needed to run to reproduce the results. See the NeurIPS code and data submission guidelines (\url{https://neurips.cc/public/guides/CodeSubmissionPolicy}) for more details.
        \item The authors should provide instructions on data access and preparation, including how to access the raw data, preprocessed data, intermediate data, and generated data, etc.
        \item The authors should provide scripts to reproduce all experimental results for the new proposed method and baselines. If only a subset of experiments are reproducible, they should state which ones are omitted from the script and why.
        \item At submission time, to preserve anonymity, the authors should release anonymized versions (if applicable).
        \item Providing as much information as possible in supplemental material (appended to the paper) is recommended, but including URLs to data and code is permitted.
    \end{itemize}

\item {\bf Experimental setting/details}
    \item[] Question: Does the paper specify all the training and test details (e.g., data splits, hyperparameters, how they were chosen, type of optimizer) necessary to understand the results?
    \item[] Answer: \answerYes{} % Replace by \answerYes{}, \answerNo{}, or \answerNA{}.
    \item[] Justification:  Please refer to the Implementation Details Section.
    \item[] Guidelines: 
    \begin{itemize}
        \item The answer \answerNA{} means that the paper does not include experiments.
        \item The experimental setting should be presented in the core of the paper to a level of detail that is necessary to appreciate the results and make sense of them.
        \item The full details can be provided either with the code, in appendix, or as supplemental material.
    \end{itemize}

\item {\bf Experiment statistical significance}
    \item[] Question: Does the paper report error bars suitably and correctly defined or other appropriate information about the statistical significance of the experiments?
    \item[] Answer: \answerYes{} % Replace by \answerYes{}, \answerNo{}, or \answerNA{}.
    \item[] Justification: We conduct the two-sided t-test on our results. 
    \item[] Guidelines:
    \begin{itemize}
        \item The answer \answerNA{} means that the paper does not include experiments.
        \item The authors should answer \answerYes{} if the results are accompanied by error bars, confidence intervals, or statistical significance tests, at least for the experiments that support the main claims of the paper.
        \item The factors of variability that the error bars are capturing should be clearly stated (for example, train/test split, initialization, random drawing of some parameter, or overall run with given experimental conditions).
        \item The method for calculating the error bars should be explained (closed form formula, call to a library function, bootstrap, etc.)
        \item The assumptions made should be given (e.g., Normally distributed errors).
        \item It should be clear whether the error bar is the standard deviation or the standard error of the mean.
        \item It is OK to report 1-sigma error bars, but one should state it. The authors should preferably report a 2-sigma error bar than state that they have a 96\% CI, if the hypothesis of Normality of errors is not verified.
        \item For asymmetric distributions, the authors should be careful not to show in tables or figures symmetric error bars that would yield results that are out of range (e.g., negative error rates).
        \item If error bars are reported in tables or plots, the authors should explain in the text how they were calculated and reference the corresponding figures or tables in the text.
    \end{itemize}

\item {\bf Experiments compute resources}
    \item[] Question: For each experiment, does the paper provide sufficient information on the computer resources (type of compute workers, memory, time of execution) needed to reproduce the experiments?
    \item[] Answer: \answerYes{} % Replace by \answerYes{}, \answerNo{}, or \answerNA{}.
    \item[] Justification: Please refer to Implementation Details.
    \item[] Guidelines:
    \begin{itemize}
        \item The answer \answerNA{} means that the paper does not include experiments.
        \item The paper should indicate the type of compute workers CPU or GPU, internal cluster, or cloud provider, including relevant memory and storage.
        \item The paper should provide the amount of compute required for each of the individual experimental runs as well as estimate the total compute. 
        \item The paper should disclose whether the full research project required more compute than the experiments reported in the paper (e.g., preliminary or failed experiments that didn't make it into the paper). 
    \end{itemize}
    
\item {\bf Code of ethics}
    \item[] Question: Does the research conducted in the paper conform, in every respect, with the NeurIPS Code of Ethics \url{https://neurips.cc/public/EthicsGuidelines}?
    \item[] Answer: \answerYes{} % Replace by \answerYes{}, \answerNo{}, or \answerNA{}.
    \item[] Justification: We conform to the NeurIPS Code of Ethics.
    \item[] Guidelines:
    \begin{itemize}
        \item The answer \answerNA{} means that the authors have not reviewed the NeurIPS Code of Ethics.
        \item If the authors answer \answerNo, they should explain the special circumstances that require a deviation from the Code of Ethics.
        \item The authors should make sure to preserve anonymity (e.g., if there is a special consideration due to laws or regulations in their jurisdiction).
    \end{itemize}

\item {\bf Broader impacts}
    \item[] Question: Does the paper discuss both potential positive societal impacts and negative societal impacts of the work performed?
    \item[] Answer: \answerYes{} % Replace by \answerYes{}, \answerNo{}, or \answerNA{}.
    \item[] Justification: We discuss the impacts in the Introduction section.
    \item[] Guidelines:
    \begin{itemize}
        \item The answer \answerNA{} means that there is no societal impact of the work performed.
        \item If the authors answer \answerNA{} or \answerNo, they should explain why their work has no societal impact or why the paper does not address societal impact.
        \item Examples of negative societal impacts include potential malicious or unintended uses (e.g., disinformation, generating fake profiles, surveillance), fairness considerations (e.g., deployment of technologies that could make decisions that unfairly impact specific groups), privacy considerations, and security considerations.
        \item The conference expects that many papers will be foundational research and not tied to particular applications, let alone deployments. However, if there is a direct path to any negative applications, the authors should point it out. For example, it is legitimate to point out that an improvement in the quality of generative models could be used to generate Deepfakes for disinformation. On the other hand, it is not needed to point out that a generic algorithm for optimizing neural networks could enable people to train models that generate Deepfakes faster.
        \item The authors should consider possible harms that could arise when the technology is being used as intended and functioning correctly, harms that could arise when the technology is being used as intended but gives incorrect results, and harms following from (intentional or unintentional) misuse of the technology.
        \item If there are negative societal impacts, the authors could also discuss possible mitigation strategies (e.g., gated release of models, providing defenses in addition to attacks, mechanisms for monitoring misuse, mechanisms to monitor how a system learns from feedback over time, improving the efficiency and accessibility of ML).
    \end{itemize}
    
\item {\bf Safeguards}
    \item[] Question: Does the paper describe safeguards that have been put in place for responsible release of data or models that have a high risk for misuse (e.g., pre-trained language models, image generators, or scraped datasets)?
    \item[] Answer: \answerNA{} % Replace by \answerYes{}, \answerNo{}, or \answerNA{}.
    \item[] Justification: All datasets used in our work are public datasets.
    \item[] Guidelines:
    \begin{itemize}
        \item The answer \answerNA{} means that the paper poses no such risks.
        \item Released models that have a high risk for misuse or dual-use should be released with necessary safeguards to allow for controlled use of the model, for example by requiring that users adhere to usage guidelines or restrictions to access the model or implementing safety filters. 
        \item Datasets that have been scraped from the Internet could pose safety risks. The authors should describe how they avoided releasing unsafe images.
        \item We recognize that providing effective safeguards is challenging, and many papers do not require this, but we encourage authors to take this into account and make a best faith effort.
    \end{itemize}

\item {\bf Licenses for existing assets}
    \item[] Question: Are the creators or original owners of assets (e.g., code, data, models), used in the paper, properly credited and are the license and terms of use explicitly mentioned and properly respected?
    \item[] Answer: \answerYes{} % Replace by \answerYes{}, \answerNo{}, or \answerNA{}.
    \item[] Justification: We properly cite the dataset and models used in our work.
    \item[] Guidelines:
    \begin{itemize}
        \item The answer \answerNA{} means that the paper does not use existing assets.
        \item The authors should cite the original paper that produced the code package or dataset.
        \item The authors should state which version of the asset is used and, if possible, include a URL.
        \item The name of the license (e.g., CC-BY 4.0) should be included for each asset.
        \item For scraped data from a particular source (e.g., website), the copyright and terms of service of that source should be provided.
        \item If assets are released, the license, copyright information, and terms of use in the package should be provided. For popular datasets, \url{paperswithcode.com/datasets} has curated licenses for some datasets. Their licensing guide can help determine the license of a dataset.
        \item For existing datasets that are re-packaged, both the original license and the license of the derived asset (if it has changed) should be provided.
        \item If this information is not available online, the authors are encouraged to reach out to the asset's creators.
    \end{itemize}

\item {\bf New assets}
    \item[] Question: Are new assets introduced in the paper well documented and is the documentation provided alongside the assets?
    \item[] Answer: \answerYes{} % Replace by \answerYes{}, \answerNo{}, or \answerNA{}.
    \item[] Justification: We have given the descriptions in our paper.
    \item[] Guidelines:
    \begin{itemize}
        \item The answer \answerNA{} means that the paper does not release new assets.
        \item Researchers should communicate the details of the dataset\slash code\slash model as part of their submissions via structured templates. This includes details about training, license, limitations, etc. 
        \item The paper should discuss whether and how consent was obtained from people whose asset is used.
        \item At submission time, remember to anonymize your assets (if applicable). You can either create an anonymized URL or include an anonymized zip file.
    \end{itemize}

\item {\bf Crowdsourcing and research with human subjects}
    \item[] Question: For crowdsourcing experiments and research with human subjects, does the paper include the full text of instructions given to participants and screenshots, if applicable, as well as details about compensation (if any)? 
    \item[] Answer: \answerNA{} % Replace by \answerYes{}, \answerNo{}, or \answerNA{}.
    \item[] Justification: No crowdsourcing experiments are conducted in this work.
    \item[] Guidelines:
    \begin{itemize}
        \item The answer \answerNA{} means that the paper does not involve crowdsourcing nor research with human subjects.
        \item Including this information in the supplemental material is fine, but if the main contribution of the paper involves human subjects, then as much detail as possible should be included in the main paper. 
        \item According to the NeurIPS Code of Ethics, workers involved in data collection, curation, or other labor should be paid at least the minimum wage in the country of the data collector. 
    \end{itemize}

\item {\bf Institutional review board (IRB) approvals or equivalent for research with human subjects}
    \item[] Question: Does the paper describe potential risks incurred by study participants, whether such risks were disclosed to the subjects, and whether Institutional Review Board (IRB) approvals (or an equivalent approval/review based on the requirements of your country or institution) were obtained?
    \item[] Answer: \answerNA{} % Replace by \answerYes{}, \answerNo{}, or \answerNA{}.
    \item[] Justification: No potential risk is found in this work.
    \item[] Guidelines:
    \begin{itemize}
        \item The answer \answerNA{} means that the paper does not involve crowdsourcing nor research with human subjects.
        \item Depending on the country in which research is conducted, IRB approval (or equivalent) may be required for any human subjects research. If you obtained IRB approval, you should clearly state this in the paper. 
        \item We recognize that the procedures for this may vary significantly between institutions and locations, and we expect authors to adhere to the NeurIPS Code of Ethics and the guidelines for their institution. 
        \item For initial submissions, do not include any information that would break anonymity (if applicable), such as the institution conducting the review.
    \end{itemize}

\item {\bf Declaration of LLM usage}
    \item[] Question: Does the paper describe the usage of LLMs if it is an important, original, or non-standard component of the core methods in this research? Note that if the LLM is used only for writing, editing, or formatting purposes and does \emph{not} impact the core methodology, scientific rigor, or originality of the research, declaration is not required.
    %this research? 
    \item[] Answer: \answerNA{} % Replace by \answerYes{}, \answerNo{}, or \answerNA{}.
    \item[] Justification: LLM usage does not impact the core methodology, scientific rigorousness, or originality of the research.
    \item[] Guidelines:
    \begin{itemize}
        \item The answer \answerNA{} means that the core method development in this research does not involve LLMs as any important, original, or non-standard components.
        \item Please refer to our LLM policy in the NeurIPS handbook for what should or should not be described.
    \end{itemize}

\end{enumerate}

\end{document}